\def\year{2022}\relax
\documentclass[letterpaper]{article} 
\usepackage{aaai22}  
\usepackage{times}  
\usepackage{helvet}  
\usepackage{courier}  
\usepackage[hyphens]{url}  
\usepackage{graphicx} 
\usepackage{natbib}  
\usepackage{caption} 
\DeclareCaptionStyle{ruled}{labelfont=normalfont,labelsep=colon,strut=off} 
\usepackage{algorithm}
\usepackage{algpseudocode}

\usepackage{bbm}
\usepackage{amsmath}
\usepackage{relsize}
\usepackage{multirow}
\usepackage{booktabs}
\usepackage{threeparttable}
\usepackage{subcaption}
\usepackage{stmaryrd}
\usepackage{longtable}

\usepackage{amsthm}

\newtheorem{theorem}{Theorem}[section]
\newtheorem{lemma}[theorem]{Lemma}

\usepackage{newfloat}
\usepackage{listings}
\floatstyle{ruled}
\newfloat{listing}{tb}{lst}{}
\floatname{listing}{Listing}
\title{TS2Vec: Towards Universal Representation of Time Series}
\author {
    Zhihan Yue,\textsuperscript{\rm 1,2}
    Yujing Wang,\textsuperscript{\rm 1,2}
    Juanyong Duan,\textsuperscript{\rm 2}
    Tianmeng Yang,\textsuperscript{\rm 1,2}\\
    Congrui Huang,\textsuperscript{\rm 2}
    Yunhai Tong,\textsuperscript{\rm 1}
    Bixiong Xu\textsuperscript{\rm 2}
}
\affiliations {
    \textsuperscript{\rm 1} Peking University, 
    \textsuperscript{\rm 2} Microsoft\\
    \{zhihan.yue,youngtimmy,yhtong\}@pku.edu.cn\\ \{yujwang,juanyong.duan,conhua,bix\}@microsoft.com
}

\usepackage{bibentry}

\begin{document}

\maketitle

\begin{abstract}
This paper presents TS2Vec, a universal framework for learning representations of time series in an \textit{arbitrary semantic level}. Unlike existing methods, TS2Vec performs contrastive learning in a \textit{hierarchical} way over \textit{augmented context} views, which enables a robust contextual representation for each timestamp. Furthermore, to obtain the representation of an arbitrary sub-sequence in the time series, we can apply a simple aggregation over the representations of corresponding timestamps. We conduct extensive experiments on time series classification tasks to evaluate the quality of time series representations. As a result, TS2Vec achieves significant improvement over existing SOTAs of unsupervised time series representation on 125 UCR datasets and 29 UEA datasets. The learned timestamp-level representations also achieve superior results in time series forecasting and anomaly detection tasks. A linear regression trained on top of the learned representations outperforms previous SOTAs of time series forecasting. Furthermore, we present a simple way to apply the learned representations for unsupervised anomaly detection, which establishes SOTA results in the literature. The source code is publicly available at \url{https://github.com/yuezhihan/ts2vec}.
\end{abstract}

\section{Introduction}

Time series plays an important role in various industries such as financial markets, demand forecasting, and climate modeling. Learning universal representations for time series is a fundamental but challenging problem. Many studies~\cite{TNC,TLoss,RWS} focused on learning \textit{instance-level} representations, which described the whole segment of the input time series and have showed great success in tasks like clustering and classification. In addition, recent works \cite{TSTCC,TLoss} employed the contrastive loss to learn the inherent structure of time series. However, there are still notable limitations in existing methods.

First, instance-level representations may not be suitable for tasks that need fine-grained representations, for example, time series forecasting and anomaly detection. In such kinds of tasks, one needs to infer the target at a specific timestamp or sub-series, while a coarse-grained representation of the whole time series is insufficient to achieve satisfied performance.

Second, few of the existing methods distinguish the multi-scale contextual information with different granularities. For example, TNC~\cite{TNC} discriminates among segments with a constant length. T-Loss~\cite{TLoss} uses random sub-series from the original time series as positive samples. However, neither of them featurizes time series at different scales to capture scale-invariant information, which is essential to the success of time series tasks. Intuitively, multi-scale features may provide different levels of semantics and improve the generalization capability of learned representations.

Third, most existing methods of unsupervised time series representation are inspired by experiences in CV and NLP domains, which have strong inductive bias such as transformation-invariance and cropping-invariance. However, those assumptions are not always applicable in modeling time series.
For example, cropping is a frequently used augmentation strategy for images. However, the distributions and semantics of time series may change over time, and a cropped sub-sequence is likely to have a distinct distribution against the original time series.


To address these issues, this paper proposes a universal contrastive learning framework called TS2Vec, which enables the representation learning of time series in \textit{all semantic levels}. It \textit{hierarchically} discriminates positive and negative samples at instance-wise and temporal dimensions; and for an arbitrary sub-series, its overall representation can be obtained by a max pooling over the corresponding timestamps. This enables the model to capture contextual information at multiple resolutions for the temporal data and generate fine-grained representations for any granularity. Moreover, the contrasting objective in TS2Vec is based on \textit{augmented context} views, that is, representations of the same sub-series in two augmented contexts should be consistent. In this way, we obtain a robust contextual representation for each sub-series without introducing unappreciated inductive bias like transformation- and cropping-invariance. 
\begin{figure*}[!htb]
  \centering
    \includegraphics[width=0.72\linewidth]{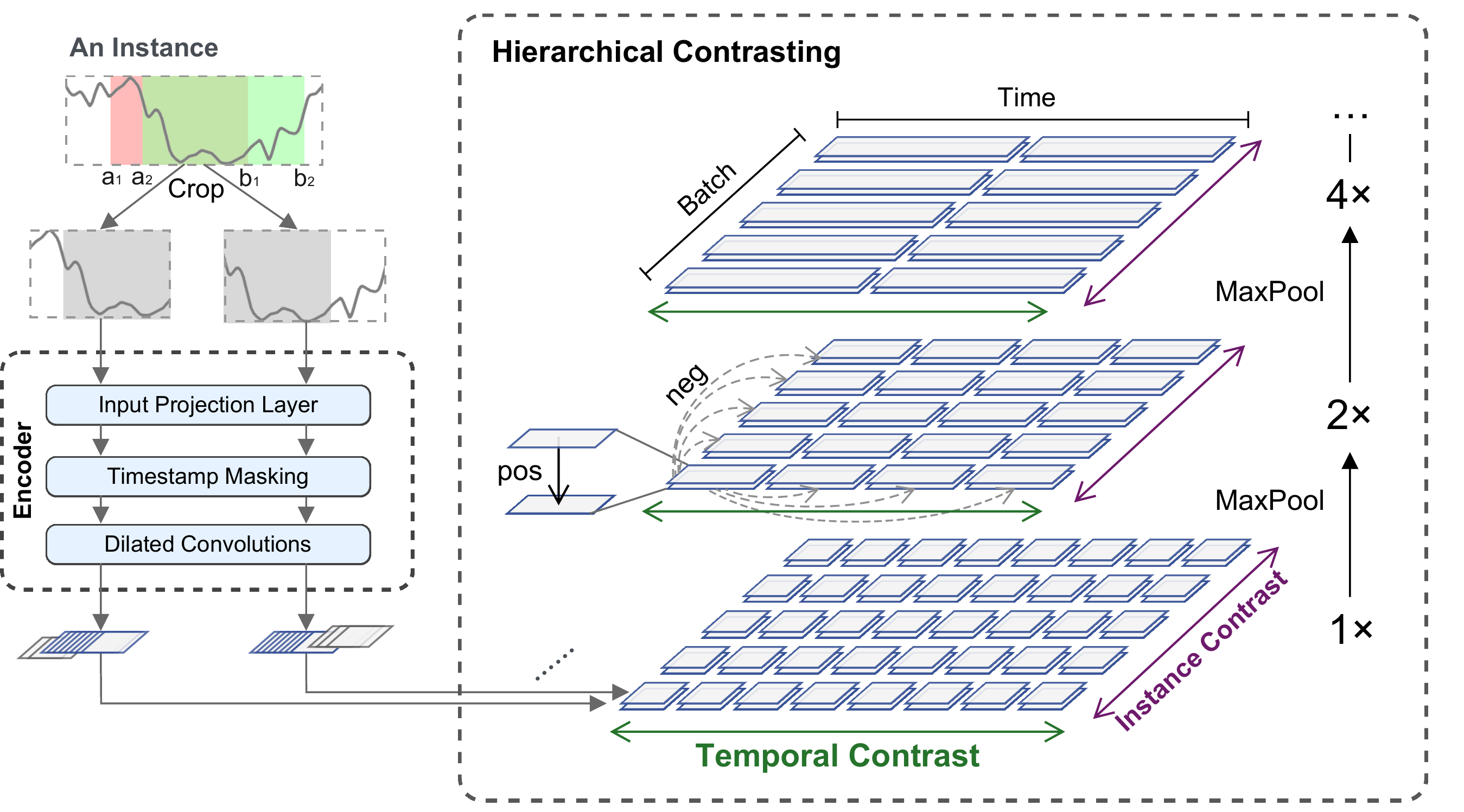}
  \caption{The proposed architecture of TS2Vec. Although this figure shows a univariate time series as the input example, the framework supports multivariate input. Each parallelogram denotes the representation vector on a timestamp of an instance.}\label{img:arch}
\end{figure*}

We conduct extensive experiments on multiple tasks to prove the effectiveness of our method. The results of time series classification, forecasting and anomaly detection tasks validate that the learned representations of TS2Vec are general and effective. 

The major contributions of this paper are summarized as follows:
\begin{itemize}
    \item We propose TS2Vec, a unified framework that learns contextual representations for arbitrary sub-series at various semantic levels. To the best of our knowledge, this is the first work that provides a flexible and universal representation method for all kinds of tasks in the time series domain, including but not limited to time series classification, forecasting and anomaly detection. 
     \item To address the above goal, we leverage two novel designs in the constrictive learning framework. First, we use a \textit{hierarchical contrasting} method in both instance-wise and temporal dimensions to capture multi-scale contextual information. Second, we propose \textit{contextual consistency} for positive pair selection. Different from previous state-of-the-arts, it is more suitable for time series data with diverse distributions and scales. Extensive analyses demonstrate the robustness of TS2Vec for time series with missing values, and the effectiveness of both hierarchical contrasting and contextual consistency are verified by ablation study.
    \item TS2Vec outperforms existing SOTAs on three benchmark time series tasks, including classification, forecasting, and anomaly detection. For example, our method improves an average of 2.4\% accuracy on 125 UCR datasets and 3.0\% on 29 UEA datasets compared with the best SOTA of unsupervised representation on classification tasks.
\end{itemize}

\section{Method}

\subsection{Problem Definition}
Given a set of time series $\mathcal{X}=\{x_1, x_2, \cdots, x_N\}$ of $N$ instances, the goal is to learn a nonlinear embedding function $f_\theta$ that maps each $x_i$ to its representation $r_i$ that best describes itself. The input time series $x_i$ has dimension $T\times F$, where $T$ is the sequence length and $F$ is the feature dimension. The representation $r_i=\{r_{i,1}, r_{i,2}, \cdots, r_{i,T}\}$ contains representation vectors $r_{i,t}\in\mathbbm{R}^K$ for each timestamp $t$, where $K$ is the dimension of representation vectors.


\subsection{Model Architecture}
The overall architecture of TS2Vec is shown in Figure~\ref{img:arch}. We randomly sample two overlapping subseries from an input time series $x_i$, and encourage consistency of contextual representations on the common segment. Raw inputs are fed into the encoder which is optimized jointly with temporal contrastive loss and instance-wise contrastive loss. The total loss is summed over multiple scales in a hierarchical framework.

The encoder $f_\theta$ consists of three components, including an input projection layer, a timestamp masking module, and a dilated CNN module. For each input $x_i$, the input projection layer is a fully connected layer that maps the observation $x_{i,t}$ at timestamp $t$ to a high-dimensional latent vector $z_{i,t}$. The timestamp masking module masks latent vectors at randomly selected timestamps to generate an augmented context view. Note that we mask latent vectors rather than raw values because the value range for time series is possibly unbounded and it is impossible to find a special token for raw data. We will further prove the feasibility of this design in the appendix. 

A dilated CNN module with ten residual blocks is then applied to extract the contextual representation at each timestamp. Each block contains two 1-D convolutional layers with a dilation parameter ($2^l$ for the $l$-th block). The dilated convolutions enable a large receptive field for different domains~\cite{TCN}. In the experimental section, we will demonstrate its effectiveness on various tasks and datasets.

\begin{figure}
  \centering
  \includegraphics[width=\linewidth]{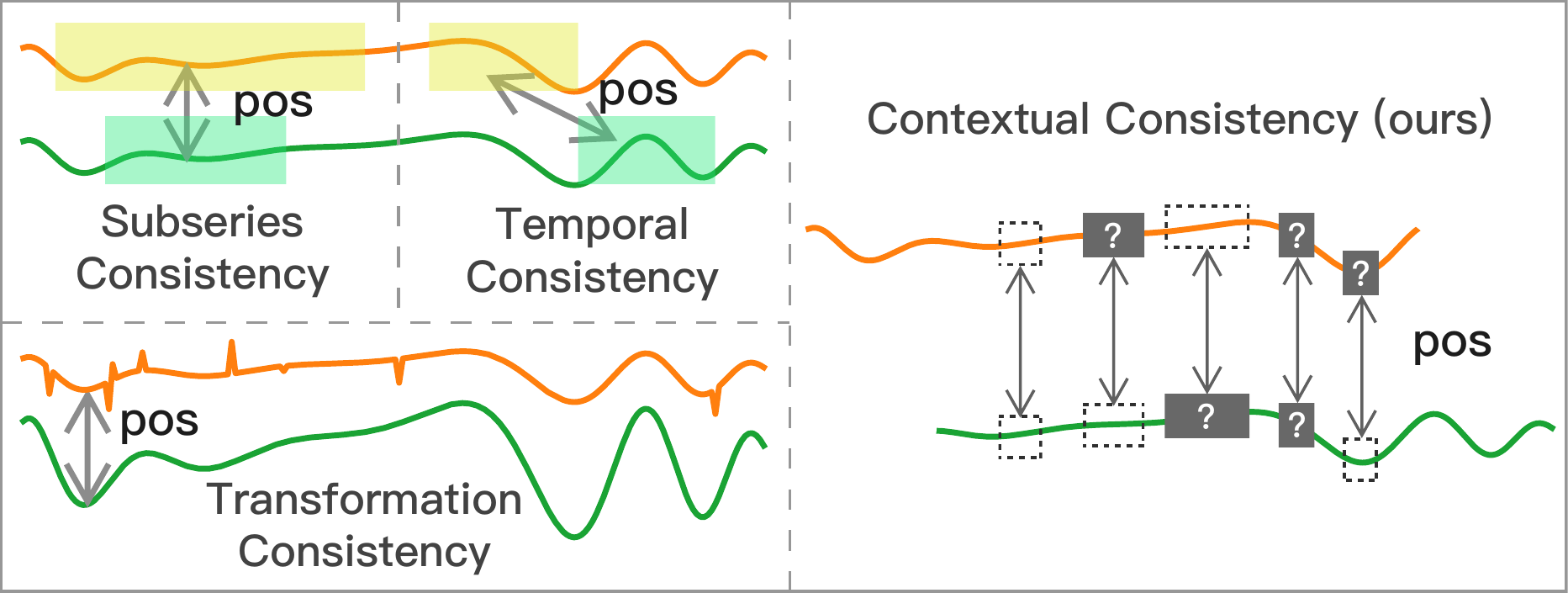}
  \caption{Positive pair selection strategies.}\label{fig:pos-pair}
\end{figure}

\begin{figure}
  \begin{subfigure}{0.49\linewidth}
    \includegraphics[width=\linewidth]{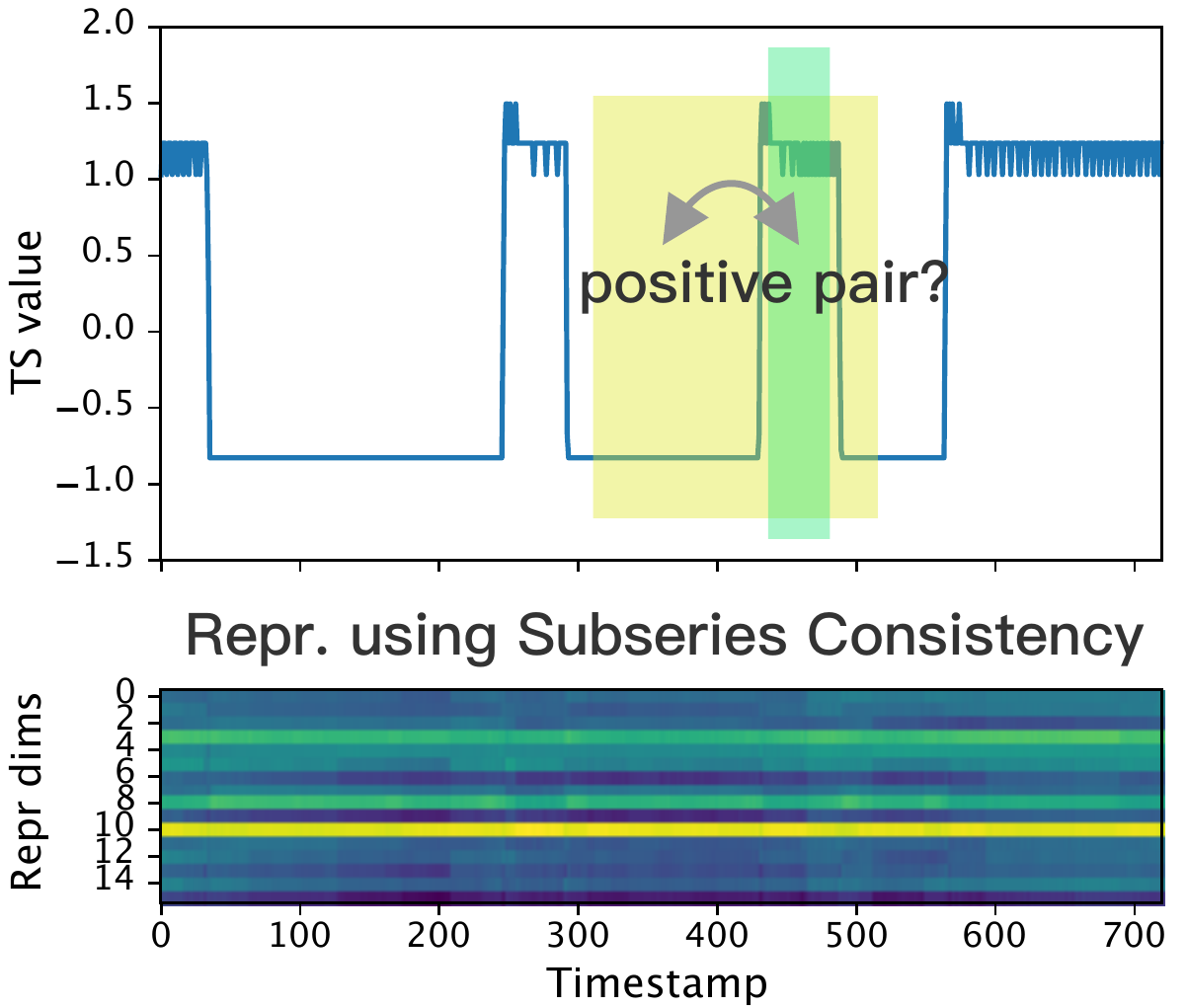}
    \caption{Level shifts.} \label{fig:case-level-shift}
  \end{subfigure}\hfill
  \begin{subfigure}{0.49\linewidth}
    \includegraphics[width=\linewidth]{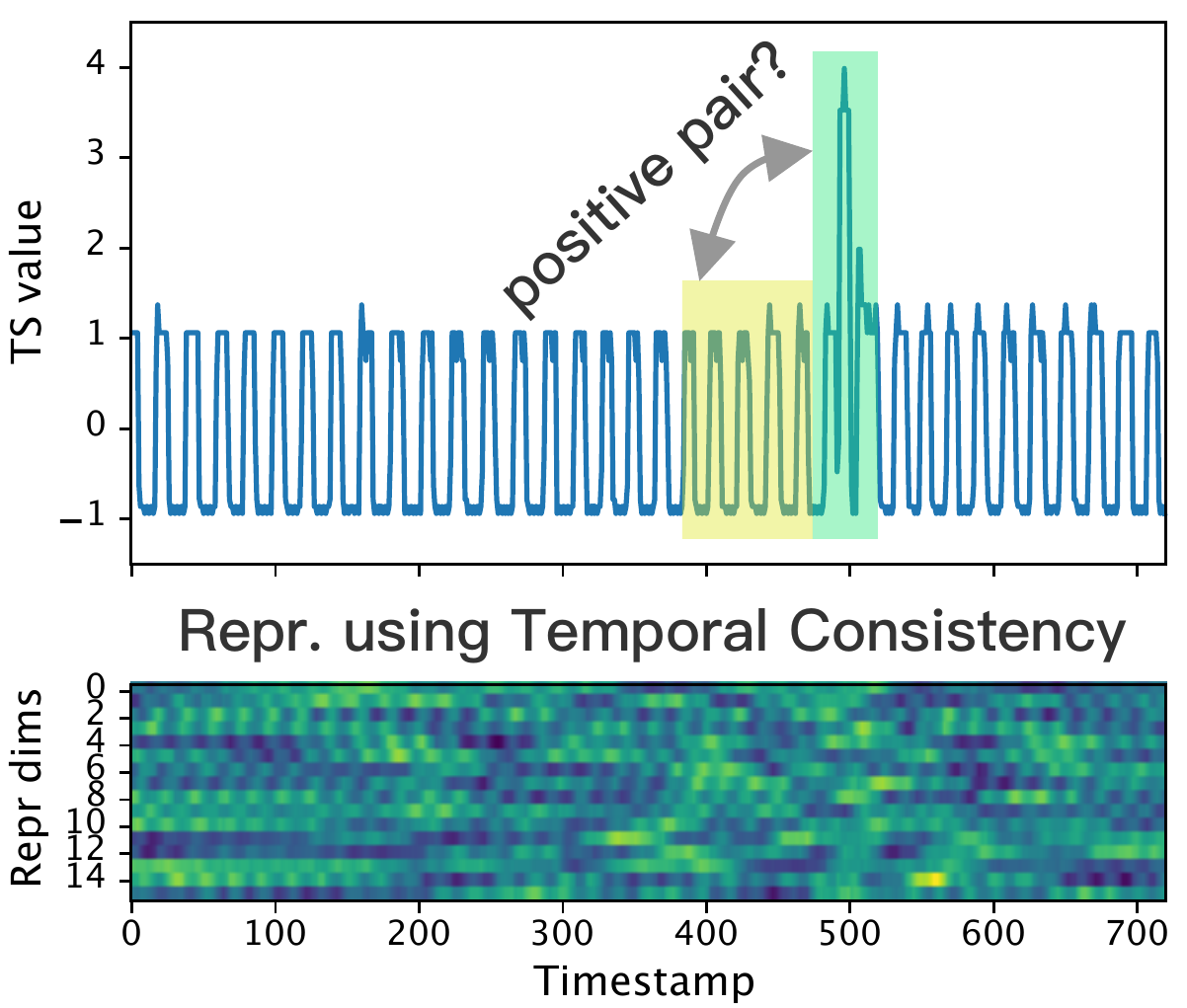}
    \caption{Anomalies.} \label{fig:case-anomaly}
  \end{subfigure}%
  \caption{Two typical cases of the distribution change of time series, with the heatmap visualization of the learned representations over time using subseries consistency and temporal consistency respectively.}\label{fig:pos-vis}
\end{figure}

\subsection{Contextual Consistency}\label{section:consistency}

The construction of positive pairs is essential in contrastive learning. Previous works have adopted various selection strategies (Figure~\ref{fig:pos-pair}), which are summarized as follows:
\begin{itemize}
    \item \emph{Subseries consistency}~\cite{TLoss} encourages the representation of a time series to be closer its sampled subseries. 
    \item \emph{Temporal consistency}~\cite{TNC} enforces the local smoothness of representations by choosing adjacent segments as positive samples.
    \item \emph{Transformation consistency}~\cite{TSTCC} augments input series by different transformations, such as scaling, permutation, etc., encouraging the model to learn transformation-invariant representations.
\end{itemize}

However, the above strategies are based on strong assumptions of data distribution and may be not appropriate for time series data. 
For example, subseries consistency is vulnerable when there exist level shifts (Figure~\ref{fig:case-level-shift}) and temporal consistency may introduce false positive pair when anomalies occur (Figure~\ref{fig:case-anomaly}). In these two figures, the green and yellow parts have different patterns, but previous strategies consider them as similar ones. To overcome this issue, we propose a new strategy, \textit{contextual consistency}, which treats the representations at the same timestamp in two augmented contexts as positive pairs. A context is generated by applying timestamp masking and random cropping on the input time series. The benefits are two-folds. First, masking and cropping do not change the magnitude of the time series, which is important to time series. 
Second, they also improve the robustness of learned representations by forcing each timestamp to reconstruct itself in distinct contexts.

\paragraph{Timestamp Masking} We randomly mask the timestamps of an instance to produce a new context view. Specifically, it masks the latent vector $z_i=\{z_{i,t}\}$ after the Input Projection Layer along the time axis with a binary mask $m\in\{0, 1\}^T$, the elements of which are independently sampled from a Bernoulli distribution with $p=0.5$. The masks are independently sampled in every forward pass of the encoder. 


\paragraph{Random Cropping}
Random cropping is also adopted to generate new contexts. For any time series input $x_i \in \mathbbm{R}^{T\times F}$, TS2Vec randomly samples two overlapping time segments $[a_1, b_1]$, $[a_2, b_2]$ such that $0 < a_1 \leq a_2 \leq b_1 \leq b_2 \leq T$. The contextual representations on the overlapped segment $[a_2, b_1]$ should be consistent for two context reviews. We show in the appendix that random cropping helps learn position-agnostic representations and avoids representation collapse.
Timestamp masking and random cropping are only applied in the training phase. 

\subsection{Hierarchical Contrasting}\label{section:hier}

In this section, we propose the hierarchical contrastive loss that forces the encoder to learn representations at various scales. The steps of calculation is summarized in Algorithm~\ref{hier-algo}. Based on the timestamp-level representation, we apply max pooling on the learned representations along the time axis and compute Equation \ref{contrastive-loss} recursively. Especially, the contrasting at top semantic levels enables the model to learn instance-level representations.


\begin{algorithm}[ht]
\caption{Calculating the hierarchical contrastive loss}\label{hier-algo}
\begin{algorithmic}[1]
\Procedure{HierLoss}{$r,r'$}
\State $\mathcal{L}_{hier} \gets \mathcal{L}_{dual}(r, r')$;
\State $d \gets 1$;
\While{$\mathrm{time\underline{~~}{length}}(r) > 1$}
\State // \textit{The maxpool1d operates along the time axis.}
\State $r \gets \mathrm{maxpool1d}(r, \mathrm{kernel\_size}=2)$;
\State $r' \gets \mathrm{maxpool1d}(r', \mathrm{kernel\_size}=2)$;
\State $\mathcal{L}_{hier} \gets \mathcal{L}_{hier} + \mathcal{L}_{dual}(r, r')$ ;
\State $d \gets d + 1$ ;
\EndWhile
\State $\mathcal{L}_{hier} \gets \mathcal{L}_{hier} / d$ ;
\State \textbf{return} $\mathcal{L}_{hier}$
\EndProcedure
\end{algorithmic}
\end{algorithm}

The hierarchical contrasting method enables a more comprehensive representation than previous works. For example, T-Loss~\cite{TLoss} performs instance-wise contrasting only at the instance level; TS-TCC~\cite{TSTCC} applies instance-wise contrasting only at the timestamp level; TNC~\cite{TNC} encourages temporal local smoothness in a specific level of granularity. These works do not encapsulate representations in different levels of granularity like TS2Vec.

\begin{table*}
  \centering
  \scalebox{0.9}{
  \begin{tabular}{lcccccc}
  \toprule
     & \multicolumn{3}{c}{125 UCR datasets} & \multicolumn{3}{c}{29 UEA datasets} \\
     \cmidrule(r){2-4} \cmidrule(r){5-7}
    Method & Avg. Acc. & Avg. Rank & Training Time (hours) & Avg. Acc. & Avg. Rank & Training Time (hours) \\
    \midrule
    DTW & 0.727 & 4.33 & -- & 0.650 & 3.74 & -- \\
    TNC & 0.761 & 3.52 & 228.4 & 0.677 & 3.84 & 91.2 \\
    TST & 0.641 & 5.23 & 17.1 & 0.635 & 4.36 & 28.6 \\
    TS-TCC & 0.757 & 3.38 & 1.1 & 0.682 & 3.53 & 3.6 \\
    T-Loss & 0.806 & 2.73 & 38.0 & 0.675 & 3.12 & 15.1 \\
    TS2Vec & \textbf{0.830 (+2.4\%)} & \textbf{1.82} & \textbf{0.9} & \textbf{0.712 (+3.0\%)} & \textbf{2.40} & \textbf{0.6}\\
    \bottomrule
  \end{tabular}
  }
  \caption{Time series classification results compared to other time series representation methods. The representation dimensions of TS2Vec, T-Loss, TS-TCC, TST and TNC are all set to 320 and under SVM evaluation protocol for fair comparison.}
  \label{cls-perf}
\end{table*}



To capture contextual representations of time series, we leverage both instance-wise and temporal contrastive losses jointly to encode time series distribution. The loss functions are applied to all granularity levels in the hierarchical contrasting model.

\paragraph{Temporal Contrastive Loss} To learn discriminative representations over time, TS2Vec takes the representations at the same timestamp from two views of the input time series as positives, while those at different timestamps from the same time series as negatives. Let $i$ be the index of the input time series sample and $t$ be the timestamp. Then $r_{i, t}$ and $r'_{i, t}$ denote the representations for the same timestamp $t$ but from two augmentations of $x_i$. The temporal contrastive loss for the $i$-th time series at timestamp $t$ can be formulated as
\begin{equation}
\mathsmaller{\ell_{temp}^{(i, t)}\!=\!-\!\log\!\frac{\exp(r_{i,t}\cdot r'_{i,t})}{\sum_{t'\in\Omega} \left(\exp(r_{i,t} \cdot r'_{i,t'}) + \mathbbm{1}_{[t\neq t']} \exp(r_{i,t} \cdot r_{i,t'})\right)},}
\end{equation}
where $\Omega$ is the set of timestamps within the overlap of the two subseries, and $\mathbbm{1}$ is the indicator function. 

\paragraph{Instance-wise Contrastive Loss} 

The instance-wise contrastive loss indexed with $(i,t)$ can be formulated as
\begin{equation}
\mathsmaller{\ell_{inst}^{(i,t)} = - \log \frac{\exp(r_{i,t} \cdot r'_{i,t})}{\sum_{j=1}^B \left(\exp(r_{i,t} \cdot r'_{j,t}) + \mathbbm{1}_{[i\neq j]} \exp(r_{i,t} \cdot r_{j,t})\right)},}
\end{equation}
where $B$ denotes the batch size. We use representations of other time series at timestamp $t$ in the same batch as negative samples.

The two losses are complementary to each other. For example, given a set of electricity consumption data from multiple users, instance contrast may learn the user-specific characteristics, while temporal contrast aims to mine the dynamic trends over time. The overall loss is defined as
\begin{equation}\label{contrastive-loss}
    \mathcal{L}_{dual} = \frac{1}{NT}\sum_{i}\sum_{t}\left(\ell_{temp}^{(i,t)} + \ell_{inst}^{(i,t)}\right).
\end{equation}

\section{Experiments}
\label{experiment}

In this section, we evaluate the learned representations of TS2Vec on time series classification, forecasting, and anomaly detection. Detailed experimental settings are presented in the appendix. 








\begin{figure}
  \centering
    \includegraphics[width=0.98\linewidth]{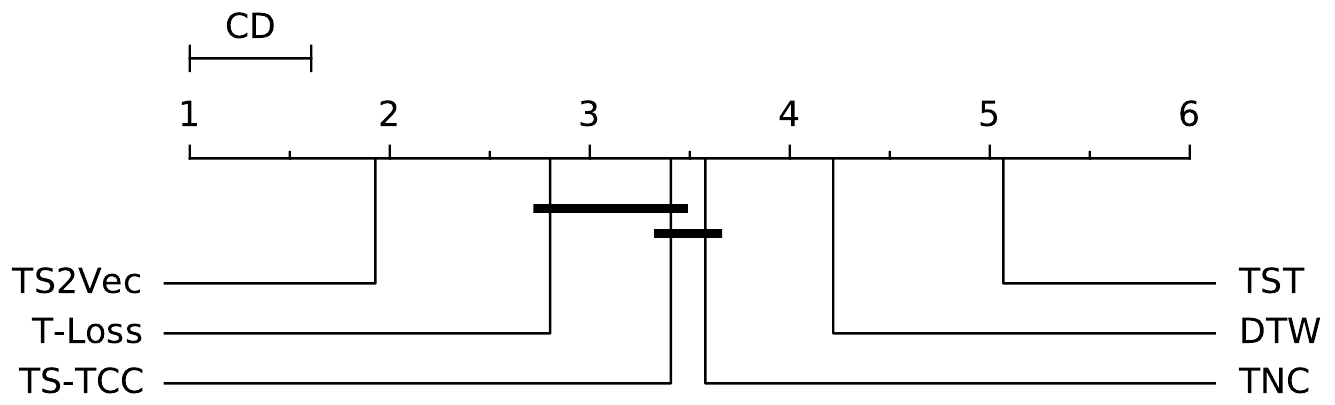}
  \caption{Critical Difference (CD) diagram of representation learning methods on time series classification tasks with a confidence level of 95\%.} \label{fig:rank}
\end{figure}

\subsection{Time Series Classification}
For classification tasks, the classes are labeled on the entire time series (instance). Therefore we require the instance-level representations, which can be obtained by max pooling over all timestamps. We then follow the same protocol as T-Loss~\cite{TLoss} where an SVM classifier with RBF kernel is trained on top of the instance-level representations to make predictions. 

We conduct extensive experiments on time series classification to evaluate the instance-level representations, compared with other SOTAs of unsupervised time series representation, including T-Loss, TS-TCC~\cite{TSTCC}, TST~\cite{TST} and TNC~\cite{TNC}. The UCR archive~\cite{UCRArchive2018} and UEA archive~\cite{UEAArchive} are adopted for evaluation. There are 128 univariate datasets in UCR and 30 multivariate datasets in UEA. Note that TS2Vec works on all UCR and UEA datasets, and full results of TS2Vec on all datasets are provided in the appendix.

The evaluation results are summarized in Table~\ref{cls-perf}. TS2Vec achieves substantial improvement compared to other representation learning methods on both UCR and UEA datasets. In particular, TS2Vec improves an average of 2.4\% classification accuracy on 125 UCR datasets and 3.0\% on 29 UEA datasets. Critical Difference diagram~\cite{CDDiagram} for Nemenyi tests on all datasets (including 125 UCR and 29 UEA datasets) is presented in Figure~\ref{fig:rank}, where classifiers that are not connected by a bold line are significantly different in average ranks. This validates that TS2Vec significantly outperforms other methods in average ranks. As mentioned in section \ref{section:consistency} and \ref{section:hier}, T-Loss, TS-TCC and TNC perform contrastive learning at only a certain level and impose strong inductive bias, such as transformation-invariance, to select positive pairs. TS2Vec applies hierarchical contrastive learning at different semantic levels, thus achieves better performance.



Table~\ref{cls-perf} also shows the total training time of representation learning methods with an NVIDIA GeForce RTX 3090 GPU. Among these methods, TS2Vec provides the shortest training time. Because TS2Vec applies contrastive losses across different granularities in one batch, the efficiency of representation learning has been greatly improved.

\begin{table}
  \centering
  \scalebox{0.8}{
  \setlength\tabcolsep{3pt}
  \begin{tabular}{lccccccccccccc}
  \toprule
    Dataset & H & TS2Vec & Informer & LogTrans & N-BEATS & TCN & LSTnet \\
    \midrule
    \multirow{5}*{ETTh$_1$}
    & 24 & \textbf{0.039} & 0.098 & 0.103 & 0.094 & 0.075 & 0.108 \\
    & 48 & \textbf{0.062} & 0.158 & 0.167 & 0.210 & 0.227 & 0.175 \\
    & 168 & \textbf{0.134} & 0.183 & 0.207 & 0.232 & 0.316 & 0.396 \\
    & 336 & \textbf{0.154} & 0.222 & 0.230 & 0.232 & 0.306 & 0.468 \\
    & 720 & \textbf{0.163} & 0.269 & 0.273 & 0.322 & 0.390 & 0.659 \\
    
    \midrule
    
    \multirow{5}*{ETTh$_2$}
    & 24 & \textbf{0.090} & 0.093 & 0.102 & 0.198 & 0.103 & 3.554 \\
    & 48 & \textbf{0.124} & 0.155 & 0.169 & 0.234 & 0.142 & 3.190 \\
    & 168 & \textbf{0.208} & 0.232 & 0.246 & 0.331 & 0.227 & 2.800 \\
    & 336 & \textbf{0.213} & 0.263 & 0.267 & 0.431 & 0.296 & 2.753 \\
    & 720 & \textbf{0.214} & 0.277 & 0.303 & 0.437 & 0.325 & 2.878 \\
    
    \midrule
    
    \multirow{5}*{ETTm$_1$}
    & 24 & \textbf{0.015} & 0.030 & 0.065 & 0.054 & 0.041 & 0.090  \\
    & 48 & \textbf{0.027} & 0.069 & 0.078 & 0.190 & 0.101 & 0.179 \\
    & 96 & \textbf{0.044} & 0.194 & 0.199 & 0.183 & 0.142 & 0.272 \\
    & 288 & \textbf{0.103} & 0.401 & 0.411 & 0.186 & 0.318 & 0.462 \\
    & 672 & \textbf{0.156} & 0.512 & 0.598 & 0.197 & 0.397 & 0.639 \\

    \midrule
    
    \multirow{5}*{Electric.}
    & 24 & 0.260 & \textbf{0.251} & 0.528 & 0.427 & 0.263 & 0.281 \\
    & 48 & \textbf{0.319} & 0.346 & 0.409 & 0.551 & 0.373 & 0.381 \\
    & 168 & \textbf{0.427} & 0.544 & 0.959 & 0.893 & 0.609 & 0.599 \\
    & 336 & \textbf{0.565} & 0.713 & 1.079 & 1.035 & 0.855 & 0.823 \\
    & 720 & \textbf{0.861} & 1.182 & 1.001 & 1.548 & 1.263 & 1.278 \\

    \midrule
    
    \multicolumn{2}{l}{Avg.} & \textbf{0.209} & 0.310 & 0.370 & 0.399 & 0.338 & 1.099 \\
    
    \bottomrule
  \end{tabular}
  }
  \caption{Univariate time series forecasting results on MSE.}
  \label{forecast-univar-mse}
\end{table}

\begin{figure}
  \centering
  \includegraphics[width=\linewidth]{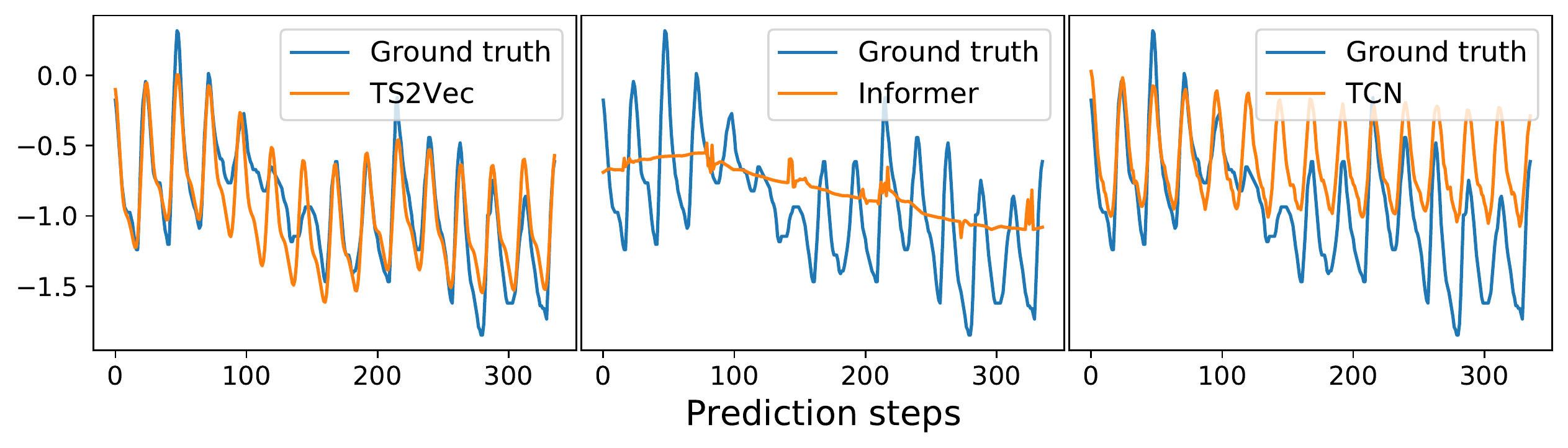}
  \caption{A prediction slice (H=336) of TS2Vec, Informer and TCN on the test set of ETTh$_2$.}\label{fig:forecast-case}
\end{figure}

\subsection{Time Series Forecasting} 

Given the last $T_l$ observations $x_{t-T_l+1}, ..., x_{t}$, time series forecasting task aims to predict the future $H$ observations $x_{t+1}, ..., x_{t+H}$. We use $r_t$, the representation of the last timestamp, to predict future observations. Specifically, we train a linear regression model with $L_2$ norm penalty that takes $r_t$ as input to directly predict future values $\hat{x}$. When $x$ is a univariate time series, $\hat{x}$ has dimension $H$. When $x$ is a multivariate time series with $F$ features, the dimension of $\hat{x}$ should be $FH$.

We compare the performance of TS2Vec and existing SOTAs on four public datasets, including three ETT datasets~\cite{Informer} and the Electricity dataset~\cite{UCI}. We apply Informer~\cite{Informer}, LogTrans~\cite{LogTrans}, LSTnet~\cite{LSTnet}, TCN~\cite{TCN} for both univariate and multivariate settings, N-BEATS~\cite{nbeats} for the univariate setting, and StemGNN~\cite{StemGNN} for the multivariate setting respectively. Follow previous works, we use MSE and MAE to evaluate the forecasting performance. 


The evaluation results on MSE for univariate forecasting are shown in Table~\ref{forecast-univar-mse}, while full forecasting results (univariate and multivariate forecasting on both MSE and MAE) are reported in the appendix due to space limitation. In general, TS2Vec establishes a new SOTA in most of the cases, where TS2Vec achieves a 32.6\% decrease of average MSE on the univariate setting and 28.2\% on the multivariate setting. Furthermore, the representations only need to be learned once for each dataset and can be directly applied to various horizons ($H$s) with linear regressions, which demonstrates the \textit{universality} of the learned representations. Figure~\ref{fig:forecast-case} presents a typical prediction slice with long-term trends and periodical patterns, comparing among the top 3 best-performing methods on univariate forecasting. In this case, Informer shows its capability to capture long-term trends but fails to capture periodical patterns. TCN successfully captures periodical patterns but fails to capture long-term trends. TS2Vec captures both characteristics, showing better predictive results than other methods.

\begin{table}
  \centering
  \scalebox{0.95}{
  \begin{tabular}{lccc}
  \toprule
    Phase & H & TS2Vec & Informer \\
    \midrule
    \multirow{5}*{Training} & 24 & 60.42 + 2.47 & 402.31 \\
    & 48 & 60.42 + 3.63 & 163.41 \\
    & 96 & 60.42 + 5.10 & 392.40 \\
    & 288 & 60.42 + 10.76 & 706.94 \\
    & 672 & 60.42 + 21.38 & 938.36 \\
    \midrule
    \multirow{5}*{Inference} & 24 & 3.01 + 0.01 & 15.91 \\
    & 48 &  3.01 + 0.02 & 4.85 \\
    & 96 &  3.01 + 0.03 & 14.57 \\
    & 288 & 3.01 + 0.10 & 21.82 \\
    & 672 & 3.01 + 0.21 & 28.49 \\
    \bottomrule
  \end{tabular}
  }
  \caption{The running time (in seconds) comparison on multivariate forecasting task on ETTm$_1$ dataset.}
  \label{forecast-time}
\end{table}

The execution time on an NVIDIA GeForce RTX 3090 GPU of the proposed method on ETTm$_1$ is presented in Table~\ref{forecast-time}, compared with Informer~\cite{Informer}, which is known as its remarkable efficiency for long time series forecasting. The training and inference time of TS2Vec are reported by two stages respectively. The training phase includes two stages: (1) learning time series representations through TS2Vec framework, (2) training a linear regressor for each $H$ on top of the learned representations. Similarly, the inference phase also includes two steps: (1) inference of representations for corresponding timestamps, (2) prediction via trained linear regressor. Note that the representation model of TS2Vec only needs to be trained once for different horizon settings.
Whether in training or inference, our method achieves superior efficiency compared to Informer.

\begin{table}
  \centering
  \scalebox{0.85}{
  \begin{tabular}{lcccccccc}
  \toprule
    & \multicolumn{3}{c}{Yahoo} & \multicolumn{3}{c}{KPI} \\
    \cmidrule(r){2-4} \cmidrule(r){5-7}
    & F$_1$ & Prec. & Rec. & F$_1$ & Prec. & Rec. \\
    \midrule
    SPOT & 0.338 & 0.269 & 0.454 & 0.217 & 0.786 & 0.126\\
    DSPOT & 0.316 & 0.241 & 0.458 & 0.521 & 0.623 & 0.447 \\
    DONUT & 0.026 & 0.013 & 0.825 & 0.347 & 0.371 & 0.326 \\
    SR & 0.563 & 0.451 & 0.747 & 0.622 & 0.647 & 0.598\\
    TS2Vec & \textbf{0.745} & 0.729 & 0.762 & \textbf{0.677} & 0.929 & 0.533 \\
    \midrule
    \multicolumn{5}{l}{\textit{Cold-start:}}\\
    FFT & 0.291 & 0.202 & 0.517 & 0.538 & 0.478 & 0.615 \\
    Twitter-AD & 0.245 & 0.166 & 0.462 & 0.330 & 0.411 & 0.276 \\
    Luminol & 0.388 & 0.254 & 0.818 & 0.417 & 0.306 & 0.650 \\
    SR & 0.529 & 0.404 & 0.765 & 0.666 & 0.637 & 0.697 \\
    TS2Vec$^\dag$ & \textbf{0.726} & 0.692 & 0.763 & \textbf{0.676} & 0.907 & 0.540 \\
    \bottomrule
  \end{tabular}
 }
  \caption{Univariate time series anomaly detection results.}
  \label{table:anomaly}
\end{table}

\subsection{Time Series Anomaly Detection}
We follow a streaming evaluation protocol \cite{SR}.
Given any time series slice $x_1, x_2, ..., x_t$ , the task of time series anomaly detection is to determine whether the last point $x_t$ is an anomaly. On the learned representations, an anomaly point may show a clear difference against normal points (Figure~\ref{fig:traj-RefrigerationDevices}). In addition, TS2Vec encourages the contextual consistency on the same timestamp of an instance. Considering this, we propose to define the anomaly score as the dissimilarity of the representations computed from masked and unmasked inputs. Specifically, on inference stage, the trained TS2Vec forwards twice for an input: for the first time, we mask out the last observation $x_t$ only; for the second time, no mask is applied. We denote the representations of the last timestamp for these two forwards as $r_t^{u}$ and $r_t^{m}$ respectively. $L_1$ distance is used to measure the anomaly score:
\begin{equation}
    \alpha_t=\left\| r_t^{u} - r_t^{m} \right\|_1.
\end{equation}
To avoid drifting, following previous works~\cite{SR}, we take the local average of the preceding $Z$ points $\overline{\alpha}_t=\frac{1}{Z}\sum_{i=t-Z}^{t-1}\alpha_i$ to adjust the anomaly score by $\alpha^{adj}_t=\frac{\alpha_t - \overline{\alpha}_t}{\overline{\alpha}_t}$. On inference, a timestamp $t$ is predicted as an anomaly point when $\alpha^{adj}_t > \mu + \beta\sigma$, where $\mu$ and $\sigma$ are the mean and standard deviation respectively of the historical scores and $\beta$ is a hyperparameter. 


We compare TS2Vec with other unsupervised methods of univariate time series anomaly detection, including FFT~\cite{FFT}, SPOT, DSPOT~\cite{POT}, Twitter-AD~\cite{twitter}, Luminol~\cite{luminol}, DONUT~\cite{DONUT} and SR~\cite{SR}. Two public datasets are used to evaluate our model. Yahoo~\cite{yahoo} is a benchmark dataset for anomaly detection, including 367 hourly sampled time series with tagged anomaly points. It converges a wide variety of anomaly types such as outliers and change-points. KPI~\cite{SR} is a competition dataset released by AIOPS Challenge. The dataset includes multiple minutely sampled real KPI curves from many Internet companies. 
The experimental settings are detailed in the appendix. 



In the normal setting, each time series sample is split into two halves according to the time order, where the first half is for unsupervised training and the second is for evaluation. However, among the baselines, Luminol, Twitter-AD and FFT do not require additional training data to start. Therefore these methods are compared under a cold-start setting, in which all the time series are for testing. In this setting, the TS2Vec encoder is trained on \textit{FordA} dataset in the UCR archive, and tested on Yahoo and KPI datasets. We denote this transferred version of our model as TS2Vec$^\dag$. We set $\beta=4$ empirically and $Z=21$ following~\cite{SR} for both settings. In the normal setting, $\mu$ and $\sigma$ of our protocol are computed using the training split for each time series, while in the cold-start setting they are computed using all historical data points before the recent point.

Table~\ref{table:anomaly} shows the performance comparison of different methods on F$_1$ score, precision and recall. In the normal setting, TS2Vec improves the F$_1$ score by 18.2\% on Yahoo dataset and 5.5\% on KPI dataset compared to the best result of baseline methods. In the cold-start setting, the F$_1$ score is improved by 19.7\% on Yahoo dataset and 1.0\% on KPI dataset than the best SOTA result. Note that our method achieves similar scores on these two settings, demonstrating the \textit{transferability} of TS2Vec from one dataset to another.

\section{Analysis}

\subsection{Ablation Study}
\begin{table}
  \centering
  \scalebox{0.88}{
  \begin{tabular}{p{5.5cm}c}
  \toprule
    & Avg. Accuracy \\
    \midrule
    \textbf{TS2Vec} & \textbf{0.829} \\
    w/o Temporal Contrast & 0.819 (-1.0\%) \\
    w/o Instance Contrast & 0.824 (-0.5\%) \\
    w/o Hierarchical Contrast & 0.812 (-1.7\%) \\
    w/o Random Cropping & 0.808 (-2.1\%) \\
    w/o Timestamp Masking & 0.820 (-0.9\%) \\
    w/o Input Projection Layer & 0.817 (-1.2\%) \\

    \midrule
    \multicolumn{2}{c}{\textit{Positive Pair Selection}}\\
    Contextual Consistency \\
    \quad $\rightarrow$ Temporal Consistency & 0.807 (-2.2\%) \\
    \quad $\rightarrow$ Subseries Consistency & 0.780 (-4.9\%) \\
    
    \midrule
    \multicolumn{2}{c}{\textit{Augmentations}}\\
    + Jitter & 0.814 (-1.5\%)\\
    + Scaling & 0.814 (-1.5\%)\\
    + Permutation & 0.796 (-3.3\%)\\
    
    \midrule
    \multicolumn{2}{c}{\textit{Backbone Architectures}}\\
    Dilated CNN\\
    \quad $\rightarrow$ LSTM & 0.779 (-5.0\%) \\
    \quad $\rightarrow$ Transformer & 0.647 (-18.2\%) \\
    \bottomrule
  \end{tabular}
  }
  \caption{Ablation results on 128 UCR datasets.}
  \label{ablation-ucr}
\end{table}
To verify the effectiveness of the proposed components in TS2Vec, a comparison between full TS2Vec and its six variants on 128 UCR datasets is shown in Table~\ref{ablation-ucr}, where (1) \textbf{w/o Temporal Contrast} removes the temporal contrastive loss, (2) \textbf{w/o Instance Contrast} removes the instance-wise contrastive loss, (3) \textbf{w/o Hierarchical Contrast} only performs contrastive learning at the lowest level, (4) \textbf{w/o Random Cropping} uses full sequence for two views rather than using random cropping, (5) \textbf{w/o Timestamp Masking} uses a mask filled with ones in training, and (6) \textbf{w/o Input Projection Layer} removes the input projection layer. The results show that all the above components of TS2Vec are indispensable.

Table~\ref{ablation-ucr} also shows the comparison among different positive pair selection strategies. We replace our proposed contextual consistency, including the timestamp masking and random cropping, into temporal consistency~\cite{TNC} and subseries consistency~\cite{TLoss}.
The temporal consistency takes the timestamps within a certain distance as positives, while the subseries consistency randomly takes two subseries for the same time series as positives.
In addition, we try to add data augmentation techniques to our method, including jitter, scaling and permutation~\cite{TSTCC}, for different views of the input time series. A performance decrease is observed after adding these augmentations. As mentioned earlier, they assume the time series data to follow some invariant assumptions which do not hold for diverse and ever-changing distributions of time series.

To justify our choice of the backbone, we replace the dilated CNN with LSTM and Transformer with a similar parameter size. 
The accuracy score decreases significantly for both cases, showing dilated CNN is an effective choice for the model architecture of time series.

\subsection{Robustness to Missing Data}

\begin{figure}
  \centering
  \includegraphics[width=0.9\linewidth]{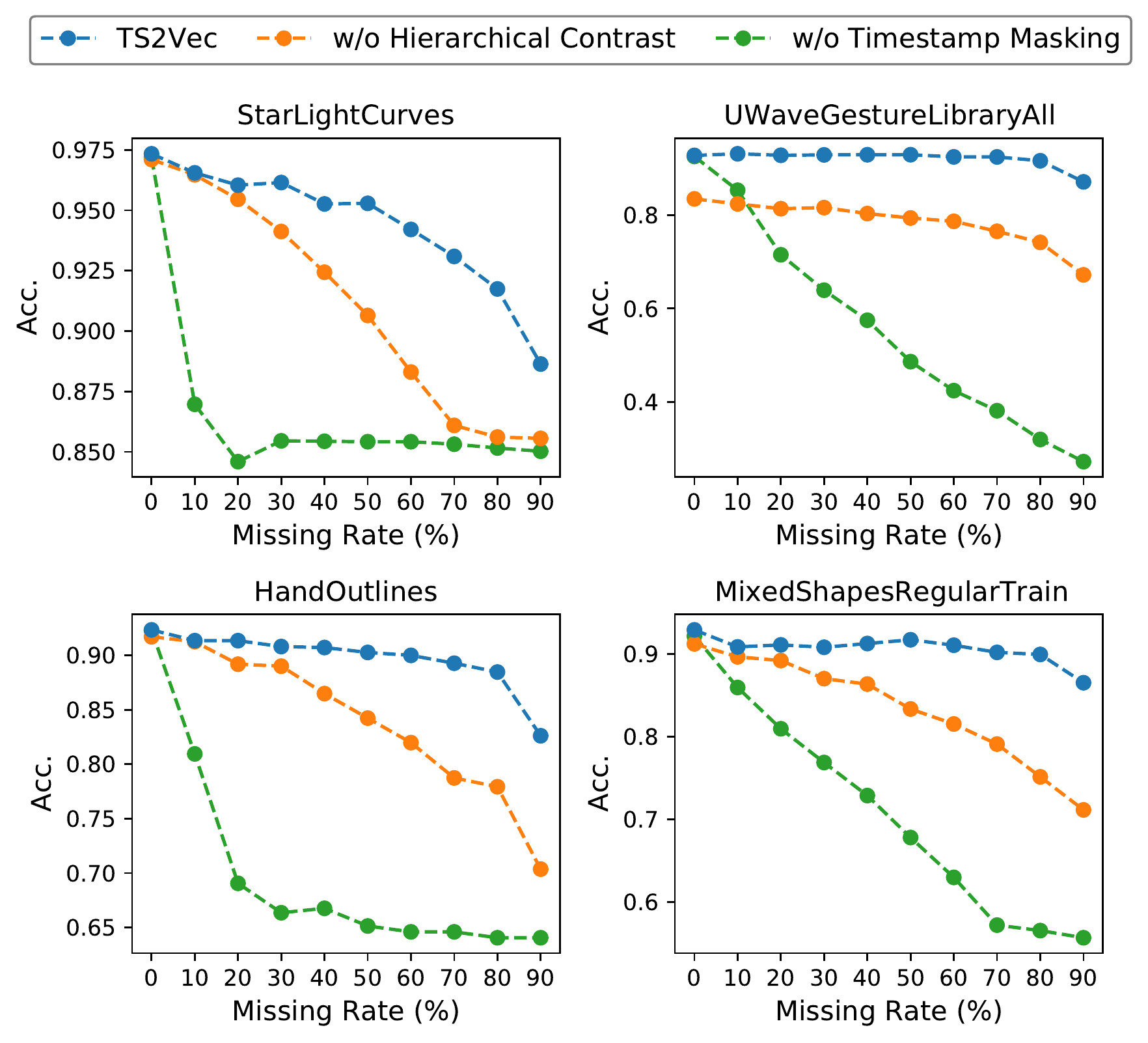}
  \caption{Accuracy scores of the top 4 largest datasets in UCR archive with respect to the rate of missing points.} \label{fig:missing}
\end{figure}

Missing data is a common occurrence for time series collected from the real world. As a universal framework, TS2Vec provides steady performance when feeding data with a large proportion of missing values, in which the proposed hierarchical contrasting and timestamp masking strategies play an important role. Intuitively, timestamp masking enables the network to infer the representations under incomplete contexts. The hierarchical contrasting brings about long-range information, which helps to predict a missing timestamp if its surrounding information is not complete.

The top 4 largest UCR datasets are selected for analysis. We randomly mask out observations for \textit{both training set and test set} with specific missing rates of timestamps. Figure~\ref{fig:missing} shows that without hierarchical contrast or timestamp masking, the classification accuracy drops rapidly with the growth of the missing rate. We also notice that the performance of \textit{w/o Hierarchical Contrast} drops dramatically as the missing rate grows, indicating the importance of long-range information for handling a large number of missing values. We can conclude that TS2Vec is extremely robust to missing points. Specifically, even with 50\% missing values, TS2Vec achieves almost the same accuracy on \textit{UWaveGestureLibraryAll}, and only 2.1\%, 2.1\% and 1.2\% accuracy decrease on \textit{StarLightCurves}, \textit{HandOutlines} and \textit{MixedShapesRegularTrain} respectively.

\subsection{Visualized Explanation}\label{visualization}

\begin{figure}
  \centering
  \captionsetup[subfigure]{font=scriptsize,labelfont=scriptsize}
  \begin{subfigure}{0.33\linewidth}
    \includegraphics[width=\linewidth]{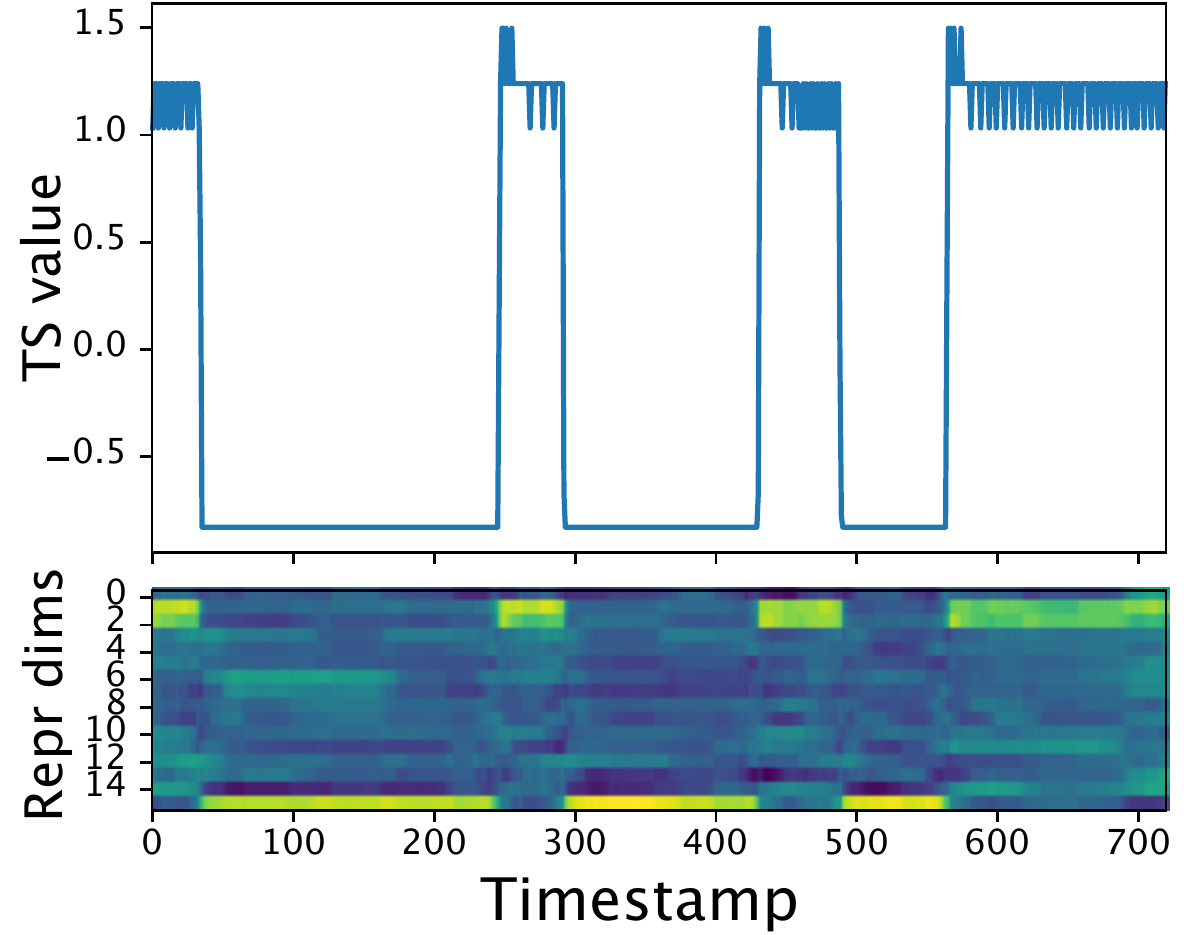}
    \caption{ScreenType.} \label{fig:traj-ScreenType}
  \end{subfigure}%
  \begin{subfigure}{0.33\linewidth}
    \includegraphics[width=\linewidth]{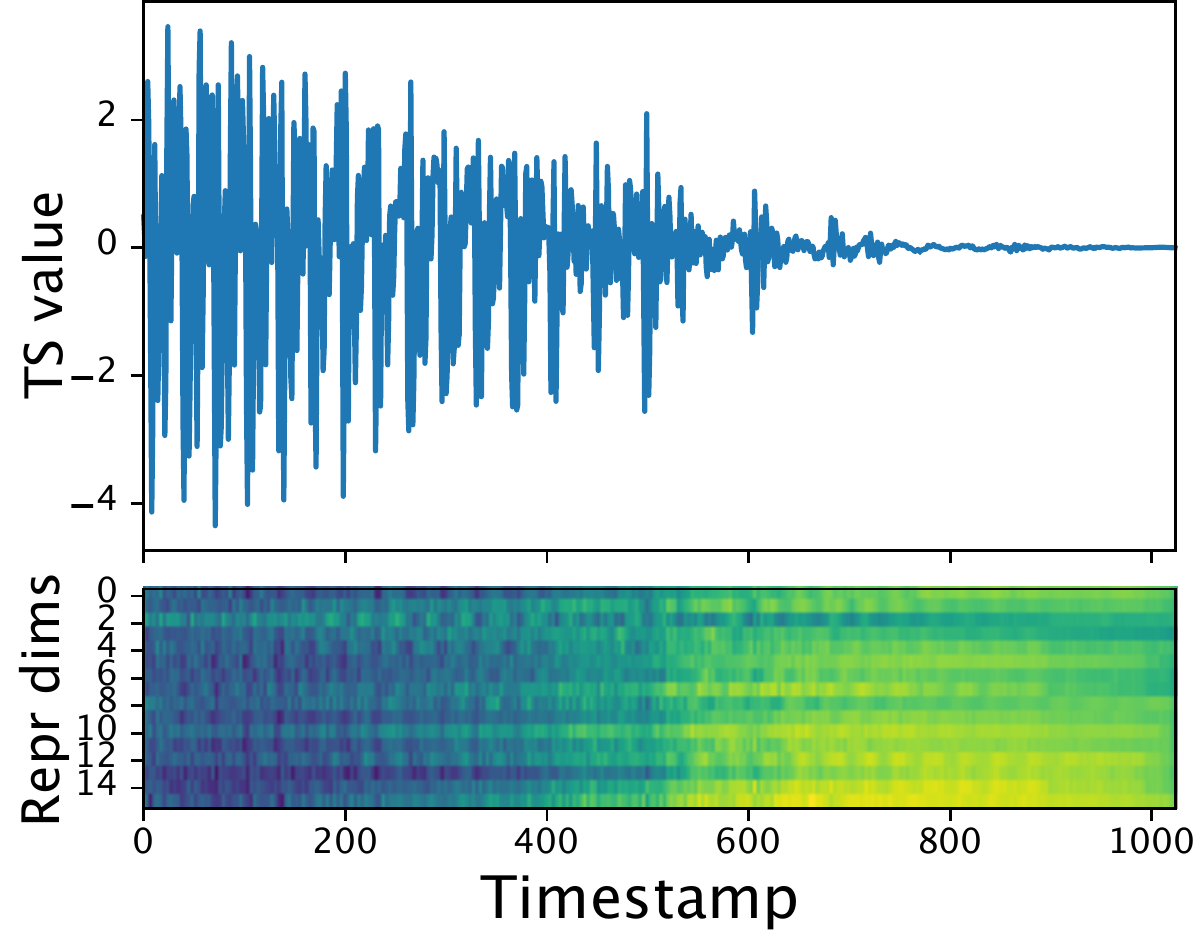}
    \caption{Phoneme.} \label{fig:traj-Phoneme}
  \end{subfigure}%
  \begin{subfigure}{0.33\linewidth}
    \includegraphics[width=\linewidth]{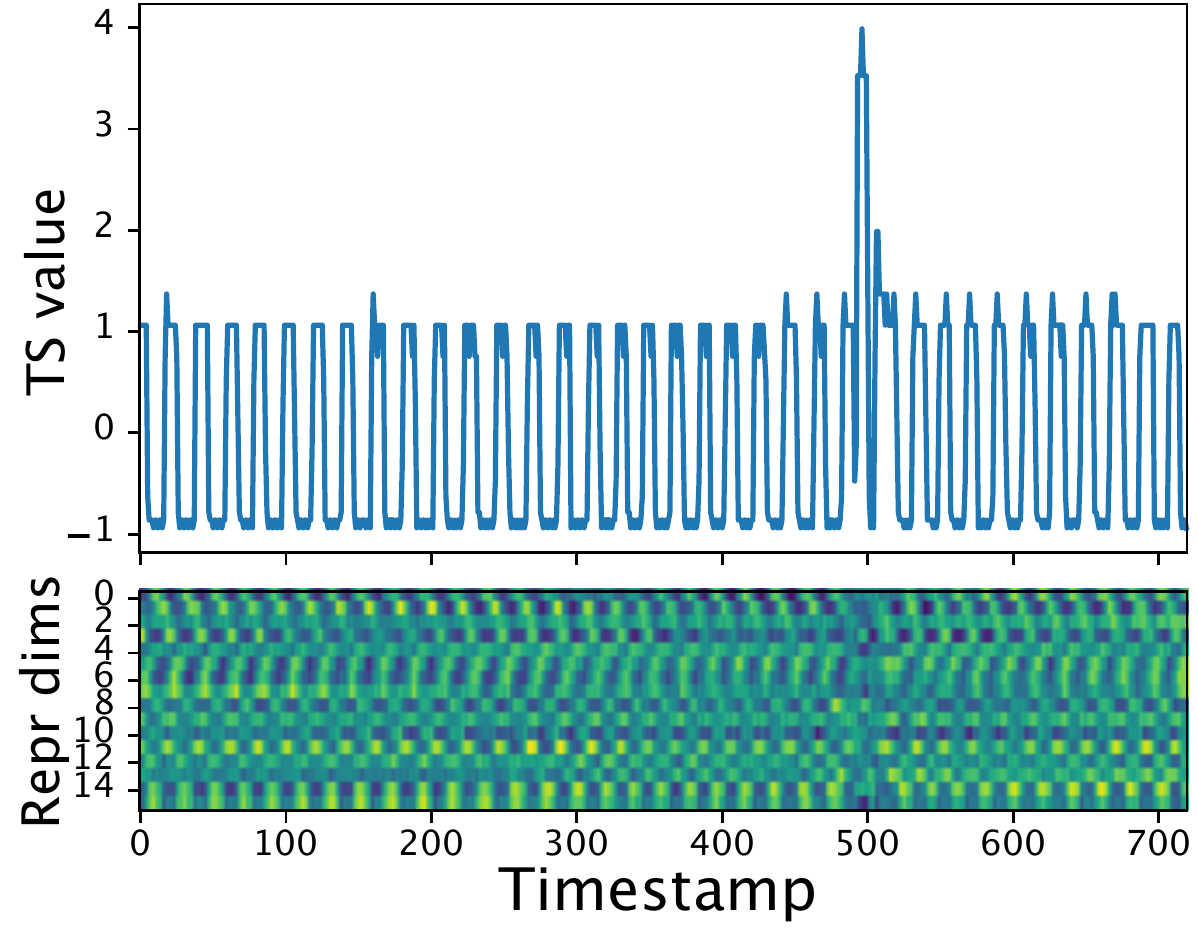}
    \caption{RefrigerationDevices.} \label{fig:traj-RefrigerationDevices}
  \end{subfigure}
  
  \caption{The heatmap visualization of the learned representations of TS2Vec over time.}\label{fig:trajs}
\end{figure}

This section visualizes the learned representations over time on three datasets from UCR archive, including \textit{ScreenType}, \textit{Phoneme} and \textit{RefrigerationDevices} datasets (Figure~\ref{fig:trajs}). We choose the first sample from the test set and select the top 16 representation dimensions with the largest variances for visualization. Figure~\ref{fig:traj-ScreenType} corresponds to a time series similar to binary digital signals, where the representation learned by TS2Vec clearly distinguishes the timestamps with high and low values respectively. 
Figure~\ref{fig:traj-Phoneme} shows an audio signal with shrinking volatility. The learned representation is able to reflect the evolving trend across timestamps. 
In Figure~\ref{fig:traj-RefrigerationDevices}, the time series has periodical patterns with a sudden spike. One can notice that the learned representations of spiked timestamps show an obvious difference from normal timestamps, demonstrating the ability of TS2Vec for capturing the change of time series distributions.

\section{Conclusion}
This paper proposes a universal representation learning framework for time series, namely TS2Vec, which applies hierarchical contrasting to learn scale-invariant representations within augmented context views. The evaluation of the learned representations on three time-series-related tasks (including time series classification, forecasting and anomaly detection) demonstrates the universality and effectiveness of TS2Vec. We also show that TS2Vec provides steady performance when feeding incomplete data, in which the hierarchical contrastive loss and timestamp masking play important roles.
Furthermore, visualization of the learned representations validates the capability of TS2Vec to capture the dynamics of time series. Ablation study proves the effectiveness of proposed components. The framework of TS2Vec is generic and has potential to be applied for other domains in our future work.

\bibliography{aaai22}

\clearpage

\appendix

\section{Related Work}

\paragraph{Unsupervised Representation of Time Series} Unsupervised representation learning has achieved good performances in computer vision~\cite{simclr,wang2020dense,xu2020hierarchical,pinheiro2020unsupervised}, natural language processing~\cite{gao2021simcse,logeswaran2018efficient} and speech recognition~\cite{baevski2020wav2vec,self-training-speech}. In the time series domain, 
SPIRAL~\cite{lei2017similarity} proposes an unsupervised method by constraining the learned representations to preserve pairwise similarities in the time domain. TimeNet~\cite{malhotra2017timenet} designs a recurrent neural network to train an encoder jointly with a decoder that reconstructs the input signal from its learned representations. RWS~\cite{RWS} constructs elaborately designed kernels to generate the vector representation of time series with an efficient approximation. TST~\cite{TST} learns a transformer-based model with a masked MSE loss. These methods are either not scalable to very long time series, or facing the challenge to model complex time series. To address this problem, T-Loss~\cite{TLoss} employs time-based negative sampling and a triplet loss to learn scalable representations for multivariate time series. TNC~\cite{TNC} leverages local smoothness of a signal to define neighborhoods in time and learns generalizable representations of time series.  TS-TCC~\cite{TSTCC} encourages consistency of different data augmentations. However, these methods only learns representations at a certain semantic level with strong assumptions on transformation-invariance, limiting their universality.

\paragraph{Time Series Forecasting} Deep learning methods including RNNs~\cite{DeepAR,wen2017mqrnn,nbeats}, CNNs~\cite{TCN,wan2019mtcn}, GNNs~\cite{StemGNN} and Transformers~\cite{LogTrans,Informer}, have been widely applied to time series forecasting tasks, which outperform classical models such as ARIMA~\cite{ariyo2014arima} and VAR~\cite{box2015time}. N-BEATS~\cite{nbeats} proposes a deep stack of fully connected layers with forward and backward residual connections for univariate times series forecasting. TCN~\cite{TCN} brings about dilated convolutions for time series forecasting and proves that dilated convolutions outperform RNNs in terms of both efficiency and predictive performance. Furthermore, LSTnet~\cite{LSTnet} combines CNNs and RNNs to capture short-term local dependencies and long-term trends simultaneously. LogTrans~\cite{LogTrans} and Informer~\cite{Informer} tackle the efficiency problem of vanilla self-attention and show remarkable performance on forecasting tasks with long sequences. In addition, graph neural networks are extensively studied in the area of multivariate time-series forecasting. For instance, StemGNN~\cite{StemGNN} models multivariate time series entirely in the spectral domain, showing competitive performance on various datasets.

\paragraph{Unsupervised Anomaly Detectors for Time Series} In the past years, statistical methods \cite{wavelet,rapiddetect} have been well-studied for time series anomaly detection. FFT~\cite{FFT} leverages fast fourier transform to detect the the areas with potential anomalies. SPOT~\cite{POT} is proposed on the basis of Extreme Value Theory, in which the threshold for anomaly scores can be automatically selected. Twitter~\cite{twitter} employs statistical methods to detect anomalies in application and system metrics. However, the traditional methods are facing challenges to be adapted to complex time series data in the real world, therefore more advanced methods are proposed. For example, DONUT~\cite{DONUT} introduces a reconstruction-based approach based on VAE. Besides, SR~\cite{SR} borrows the Spectral Residual model from visual saliency detection domain, showing outstanding performance for anomaly detection.

\section{Supplement of Method}

\begin{table}
  \centering
  \begin{tabular}{p{5.5cm}c}
  \toprule
    & Avg. Acc. \\
    \midrule
    Masking after Input Projection & \textbf{0.829} \\
    Masking before Input Projection & 0.822 \\
    w/o Timestamp Masking & 0.820 \\
    \bottomrule
  \end{tabular}
  \caption{Ablation results of the timestamp masking on 128 UCR datasets.}
  \label{masking-ucr}
\end{table}

\subsection{Input Projection Layer}

This section justifies why the timestamp masking is applied after the input projection layer. In common practice, raw inputs are masked. However, the value range for time series is possibly unbounded and it is impossible to find a special token for raw data. Intuitively, one can mask raw inputs with value $0$, but a time series may contain a segment whose values are all $0$s, which is indistinguishable from the mask token.

In TS2Vec, an input projection layer is employed before timestamp masking to project the the input vectors into higher-dimensional latent vectors:
\begin{equation}
    z_{i,t} = Wx_{i,t}+b,
\end{equation}
where $x_{i,t}\in \mathbbm{R}^{F}$, $W\in \mathbbm{R}^{F'\times F}$, $b\in \mathbbm{R}^{F'}$, and $F'>F$.

Then, we set the latent vector on each masked timestamp to $0$. It can be proved that there are a set of parameters $\{W,b\}$ for the network, such that $Wx+b\neq 0$ for any input $x\in \mathbbm{R}^{F}$ (Lemma \ref{lemma:proj}). Hence, the network has the ability to learn to distinguish the masked and unmasked inputs. Table~\ref{masking-ucr} demonstrates that applying the timestamp masking after the input projection layer achieves better performance.

\begin{lemma}\label{lemma:proj}
Given $W\in \mathbbm{R}^{F'\times F}$ and $F'>F$, there exists $b\in \mathbbm{R}^{F'}$ such that $Wx+b\neq 0$ for all $x\in \mathbbm{R}^{F}$.
\end{lemma}
\begin{proof}
Let $W=[w_1, w_2, ..., w_F]$, where $w_i$ is the column vector and $\mathrm{dim}(w_i)=F'>F$. We have $\mathrm{rank}(W)\leq F<F'=\mathrm{dim}(b)$, so $\exists b\in \mathbbm{R}^{F'}, b \not\in \mathrm{span}(W)$. Since $\exists x\in \mathbbm{R}^{F}, Wx+b=0$ iff $\exists x\in \mathbbm{R}^{F}, Wx=b$ iff $b\in \mathrm{span}(W)$, we conclude that $\exists b\in \mathbbm{R}^{F'}, \forall x\in \mathbbm{R}^{F}, Wx+b\neq 0$.
\end{proof}

\begin{figure}
  \centering
  \includegraphics[width=0.9\linewidth]{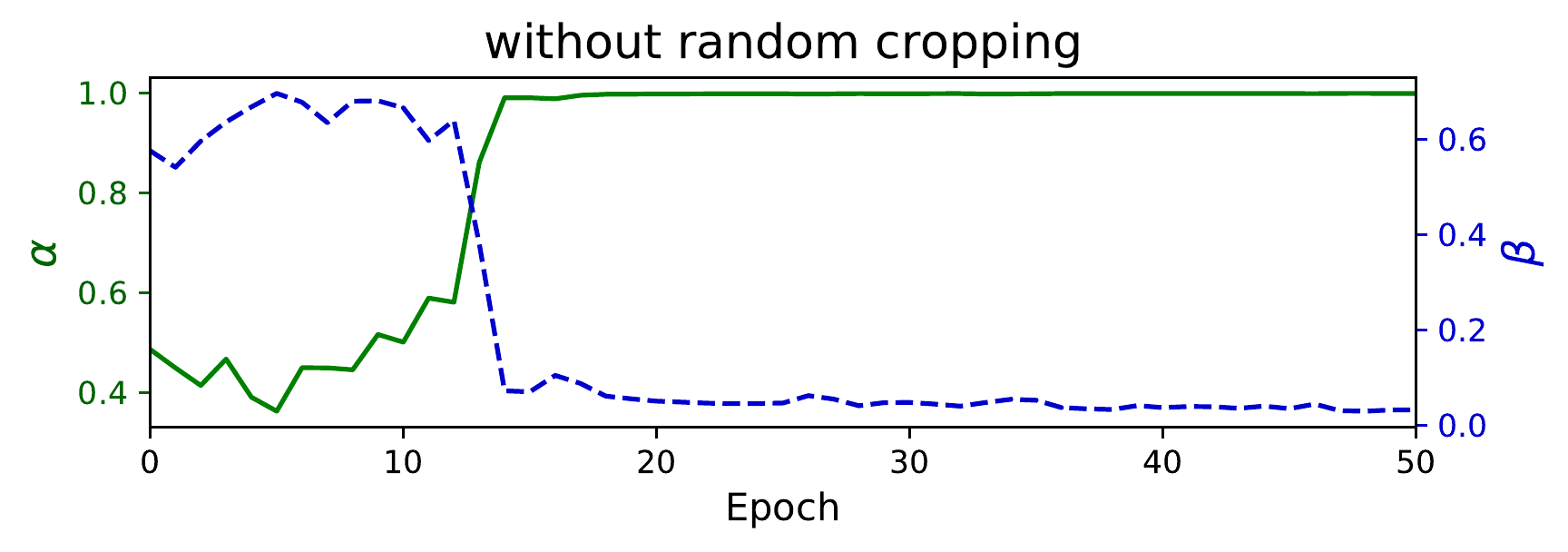}
  \includegraphics[width=0.9\linewidth]{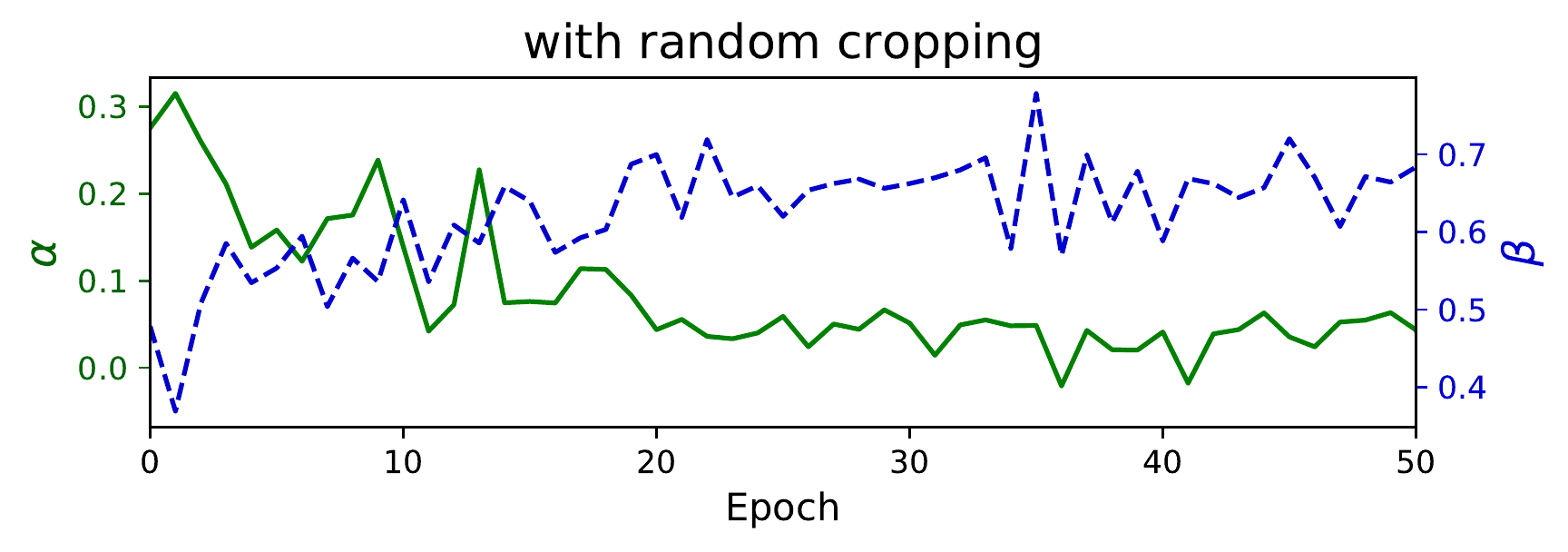}
  \caption{$\alpha$ and $\beta$ over training epochs of temporal contrastive loss on the \textit{ScreenType} dataset.}\label{fig:crop-metrics}
\end{figure}

\subsection{Random Cropping}

For any time series input $x_i \in \mathbbm{R}^{T\times F}$, TS2Vec randomly samples two overlapping time segments $[a_1, b_1]$, $[a_2, b_2]$ such that $0 < a_1 \leq a_2 \leq b_1 \leq b_2 \leq T$. The contextual representations on the overlapped segment $[a_2, b_1]$ should be consistent for two context reviews.

Random cropping is not only a complement to timestamp masking but also a critical design for learning position-agnostic representations. While performing temporal contrasting, removal of the random cropping module can result in \textit{representation collapse}. This phenomenon can be attributed to implicit positional contrasting. It has been proved that convolutional layers can implicitly encode positional information into into latent space~\cite{convposembed1, convposembed2}. Our proposed random cropping makes it impossible for the network to infer the position of a time point on the overlap. When position information is available for temporal contrasting, the model can only utilize the absolute position rather than the contextual information to reduce the temporal contrastive loss. For example, the model can always output $\mathrm{onehot}(t)$, i.e. a vector whose $t$-th element is 1 and other elements is 0, on $t$-th timestamp of the input to reach a global minimum of the loss. To demonstrate this phenomenon, we mask all value of the input series $x_i$ and take the output of the encoder as a pure representation of positional information $p_i \in \mathbbm{R}^{T\times K}$. Then we define a metric $\alpha$ to measure the similarity between the learned representation $r_i$ and the pure positional representation $p_i$ as:
\begin{equation}
    \alpha=\frac{1}{NT}\sum\nolimits_i\sum\nolimits_t\frac{r_{i,t}\cdot p_{i,t}}{\left\|r_{i,t}\right\|_2\left\|p_{i,t}\right\|_2}.
\end{equation}

Furthermore, $\beta$ is defined to measure the deviation between the representations of different samples as:
\begin{equation}
    \small{\beta=\!\frac{1}{TK}\!\sum\nolimits_t \sum\nolimits_k\!\left(\!\sqrt{\sum\nolimits_i\left(r_i-\sum\nolimits_j r_j/N\right)^2/N}\right)_{t,k}\!.}
\end{equation}

Figure \ref{fig:crop-metrics} shows that without random cropping the $\alpha$ is close to $1$ in the later stage of training, while the $\beta$ drops significantly. It indicates a representation collapse that the network only learns the positional embedding as a representation and overlooks the contextual information in this case. In contrast, with random cropping, $\beta$ keeps a relatively high level and avoids the collapse.

\section{Experimental Details}

\subsection{Data Preprocessing}

\paragraph{Normalization}
Following \citet{TLoss,Informer}, for univariate time series, we normalize datasets using z-score so that the set of observations for each dataset has zero mean and unit variance. For multivariate time series, each variable is normalized independently using z-score. For forecasting tasks, all reported metrics are calculated based on the normalized time series.

\paragraph{Variable-Length Data and Missing Observations}
For a variable-length dataset, we pad all the series to the same length. The padded values are set to \textit{NaNs}, representing the missing observations in our implementation. When an observation is missing (\textit{NaN}), the corresponding position of the mask would be set to zero.

\paragraph{Extra Features}
Following~\citet{Informer,DeepAR}, we add extra time features to the input, including minute, hour, day-of-week, day-of-month, day-of-year, month-of-year, and week-of-year (when the corresponding information is available). This is only applied for forecasting datasets, because timestamps are unavailable for time series classification datasets like UCR and UEA.

\subsection{Reproduction Details for TS2Vec}

On representation learning stage, labels and downstream tasks are assumed to be unknown, thus selecting hyperparameters for unsupervised models is challenging. For example, in supervised training, early stopping is a widely used technique based on the development performance. However, without labels, it is hard to know which epoch learns a better representation for the downstream task. Therefore, for our representation learning model, \textit{a fixed set of hyperparameters} is set empirically regardless of the downstream task, and no additional hyperparameter optimization is performed, unlike other unsupervised works such as \cite{TNC,TSTCC}.

For each task, we only use the training set to train the representation model, and apply the model to the testing set to get representations. The batch size is set to 8 by default. The learning rate is 0.001. The number of optimization iterations is set to 200 for datasets with a size less than 100,000, otherwise it is 600. The representation dimension is set to 320 following \cite{TLoss}. In the training phase, we crop a large sequence into pieces with 3,000 timestamps in each. In the encoder of TS2Vec, the linear projection layer is a fully connected layer that maps the input channels to hidden channels, where input channel size is the feature dimension, and the hidden channel size is set to 64. The dilated CNN module contains 10 hidden blocks of "GELU $\rightarrow$ DilatedConv $\rightarrow$ GELU $\rightarrow$ DilatedConv" with skip connections between adjacent blocks. For the $i$-th block, the dilation of the convolution is set to $2^i$. The kernel size is set to 3. Each hidden dilated convolution has a channel size of 64. Finally, an output residual block maps the hidden channels to the output channels, where the output channel size is the representation dimension. All experiments are conducted on a NVIDIA GeForce RTX 3090 GPU.

Changing the default hyperparameters may benefit the performance on some datasets, but worsen it on others. For example, similar to \cite{TLoss}, we find that the number of negative samples in a batch (corresponding to the batch size for TS2Vec) has a notable impact on the performance for individual datasets (see section \ref{cls-full-ucr}).

\subsection{Reproduction Details for Baselines}

For classification tasks, the results of TNC~\cite{TNC}, TS-TCC~\cite{TSTCC} and TST~\cite{TST} are based on our reproduction. For forecasting tasks, the results of TCN, StemGNN and N-BEATS for all datasets and all baselines for Electricity dataset are based on our reproduction. Other results for baselines in this paper are directly taken from~\cite{TLoss, Informer, SR}.

TNC~\cite{TNC}: TNC leverages local smoothness of a signal to define neighborhoods in time and learns generalizable representations for time series. We use the open source code from \url{https://github.com/sanatonek/TNC_representation_learning}. We use the encoder of TS2Vec rather than their original encoders (CNN and RNN) as its backbone, because we observed significant performance improvement for TNC using our encoder compared to using its original encoder on UCR and UEA datasets. This can be attributed to the adaptive receptive fields of dilated convolutions, which better fit datasets from various scales. For other settings, we refer to their default settings on waveform data.

TS-TCC~\cite{TSTCC}: TS-TCC encourages consistency of different data augmentations to learn transformation-invariant representations. We take the open source code from \url{https://github.com/emadeldeen24/TS-TCC}. The batch size is set to 8 for UEA datasets and 128 for UCR datasets. In our setting, the representation dimension of each instance is set to 320 for any baselines. Therefore, the output dimension is set to 320 and a max pooling is employed to aggregate timestamp-level representations into instance level. We refer to their configuration on HAR data for other experimental settings.

TST~\cite{TST}: TST learns a transformer-based model with a masked MSE loss. We use the open source code from \url{https://github.com/gzerveas/mvts_transformer}. The subsample factor is set to $\lceil T / 1000\rceil$ due to memory limitation. We set the representation dimension to 320. A max pooling layer is employed to aggregate timestamp-level representations into instance level, so that the representation size for an instance is 320. Other settings remain the default values in the code.

Informer~\cite{Informer}: Informer is an efficient transformer-based model for time series forecasting and is the previous SOTA on ETT datasets. We use the open source code at \url{https://github.com/zhouhaoyi/Informer2020}. For Electricity dataset, we use the following settings in reproduction: for multivariate cases with H=24,48,168,336,720, the label lengths are 48,48,168,168,336 respectively, and the sequence lengths are 48,96,168,168,336 respectively; for univariate cases with H=24,48,168,336,720, the label lengths are 48,48,336,336,336 respectively, and the sequence lengths are 48,96,336,336,336 respectively; other settings are set by default.

StemGNN~\cite{StemGNN}: StemGNN models multivariate time series entirely in the spectral domain with Graph Fourier Transform and Discrete Fourier Transform. We use the open source code from \url{https://github.com/microsoft/StemGNN}. The window size is set to 100 for H=24/48, 200 for H=96/168, 400 for H=288/336, and 800 otherwise. For other settings, we refer to the paper and default values in open-source code.

TCN~\cite{TCN}: TCN brings about dilated convolutions for time series forecasting. We take the open source code at \url{https://github.com/locuslab/TCN}. We use a stack of ten residual blocks, each of which has a hidden size of 64, following our backbone. The upper epoch limit is 100, and the learning rate is 0.001. Other settings remain the default values in the code.

LogTrans~\cite{LogTrans}: LogTrans breaks the memory bottleneck of Transformers and produces better results than canonical Transformer on time series forecasting. Due to no official code available, we use a modified version of a third-party implementation at \url{https://github.com/mlpotter/Transformer_Time_Series}. The embedding vector size is set to 256, and the kernel size for casual convolutions is 9. We stack three layers for their Transformer. We refer to the paper for other experimental settings.

LSTnet~\cite{LSTnet}: LSTnet combines CNNs and RNNs to incorporate both short-term local dependencies and long-term trends. We take the open source code at \url{https://github.com/laiguokun/LSTNet}. In our experiments, the window size is 500, the hidden channel size is 50, and the CNN filter size is 6. For other settings, we refer to the paper and the default values in code.

N-BEATS~\cite{nbeats}: N-BEATS proposes a deep stack of fully connected layers with forward and backward residual connections for univariate times series forecasting. We take the open source code at \url{https://github.com/philipperemy/n-beats}. In our experiments, two generic stacks with two blocks in each are used. We use a backcast length of 1000 and a hidden layer size of 64. We turn on the 'share\_weights\_in\_stack' option as recommended.

\subsection{Details for Benchmark Tasks}

\paragraph{Time Series Classification} For TS2Vec, the instance-level representations can be obtained by max pooling over all timestamps. To evaluate the instance-level representations on time series classification, we follow the same protocol as \citet{TLoss} where an SVM classifier with RBF kernel is trained on top of the instance-level representations. The penalty $C$ is selected using a grid search by cross-validation of the training set from a search space of $\left\{10^{i} \mid i \in \llbracket-4,4 \rrbracket\right\} \cup\{\infty\}$.

T-Loss, TS-TCC and TNC cannot handle datasets with missing observations, including \textit{DodgerLoopDay}, \textit{DodgerLoopGame} and \textit{DodgerLoopWeekend}. Besides, the result of DTW on \textit{InsectWingbeat} dataset in UEA archive is not reported. Therefore we conduct comparison over the remaining 125 UCR datasets and 29 UEA datasets in the main paper. Note that TS2Vec works on all UCR and UEA datasets, and full results of TS2Vec on all datasets are provided in section \ref{cls-full-ucr}.

\paragraph{Time Series Forecasting} To evaluate the timestamp-level representations on time series forecasting, we propose a linear protocol where a ridge regression (i.e., a linear regression model with $L_2$ regularization term $\alpha$) is trained on top of the learned representations to predict the future values. The regularization term $\alpha$ is selected using a grid search on the validation set from a search space of \{0.1, 0.2, 0.5, 1, 2, 5, 10, 20, 50, 100, 200, 500, 1000\}, while all evaluation results are reported on the test set.

We use two metrics to evaluate the forecasting performance, including $\mathrm{MSE}=\frac{1}{HF}\sum_{i=1}^H\sum_{j=1}^F(x_{t+i}^{(j)}-\hat{x}_{t+i}^{(j)})^2$ and $\mathrm{MAE}=\frac{1}{HF}\sum_{i=1}^H\sum_{j=1}^F |x_{t+i}^{(j)}-\hat{x}_{t+i}^{(j)}|$, where $x_{t+i}^{(j)}, \hat{x}_{t+i}^{(j)}$ is the observed and predicted value respectively on variable $j$ at timestamp $t+i$. The overall metrics for a dataset is the average MSE and MAE over all slices and instances.
 
ETT datasets collect 2-years power transformer data from 2 stations, including ETTh$_1$, ETTh$_2$ for 1-hour-level and ETTm$_1$ for 15-minute-level, which combines long-term trends, periodical patterns, and many irregular patterns. The Electricity dataset contains the electricity consumption data of 321 clients (instances) over 3 years. Following~\cite{Informer, LogTrans}, we resample Electricity into hourly data. The train/val/test split is 12/4/4 months for ETT datasets following~\cite{Informer} and 60\%/20\%/20\% for Electricity. For univariate forecasting tasks, the target value is set as 'oil temperature' for ETT datasets and 'MT\_001' for Electricity. To evaluate the performance of both short-term and long-term forecasting, we prolong the prediction horizon $H$ progressively, from 1 day to 30 days for hourly data and from 6 hours to 7 days for minutely data.

\paragraph{Time Series Anomaly Detection} In common practice, point-wise metrics are not concerned. It is adequate for the anomaly detector to trigger an alarm at any point in a continuous anomaly segment if the delay is not too long. Hence, we follow exactly the same evaluation strategy as~\citet{SR,DONUT}, where anomalies detected within a certain delay (7 steps for minutely data and 3 steps for hourly data) are considered correct. Specifically, if any time point in an anomaly segment is detected as an anomaly within a certain delay, all points in this segment are treated as correct, incorrect otherwise. For any normal time point, no adjustment is applied. Besides, in preprocess stage, the raw data is differenced $d$ times to avoid drifting, where $d$ is the number of unit roots determined by the ADF test.

\section{Full Results}

\subsection{Time Series Forecasting}

The evaluation results are shown in Table~\ref{forecast-univar} for univariate forecasting and Table~\ref{forecast-multivar} for multivariate forecasting. In general, TS2Vec establishes a new SOTA in most of the cases, where TS2Vec achieves a 32.6\% decrease of average MSE on the univariate setting and 28.2\% on the multivariate setting.

\subsection{Time Series Classification}\label{cls-full-ucr}

\begin{figure*}
  \centering
  \begin{subfigure}{0.3\linewidth}
    \includegraphics[width=\linewidth]{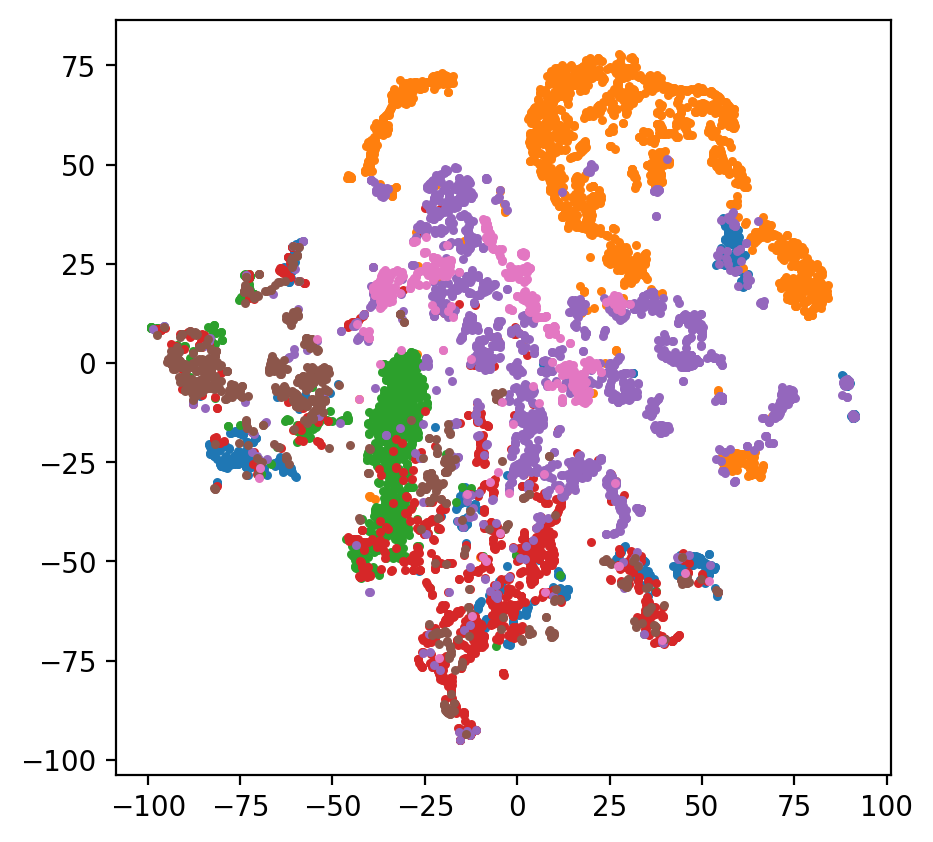}
    \caption{ElectricDevices.} \label{fig:tsne-elec}
  \end{subfigure}%
  \begin{subfigure}{0.3\linewidth}
    \includegraphics[width=\linewidth]{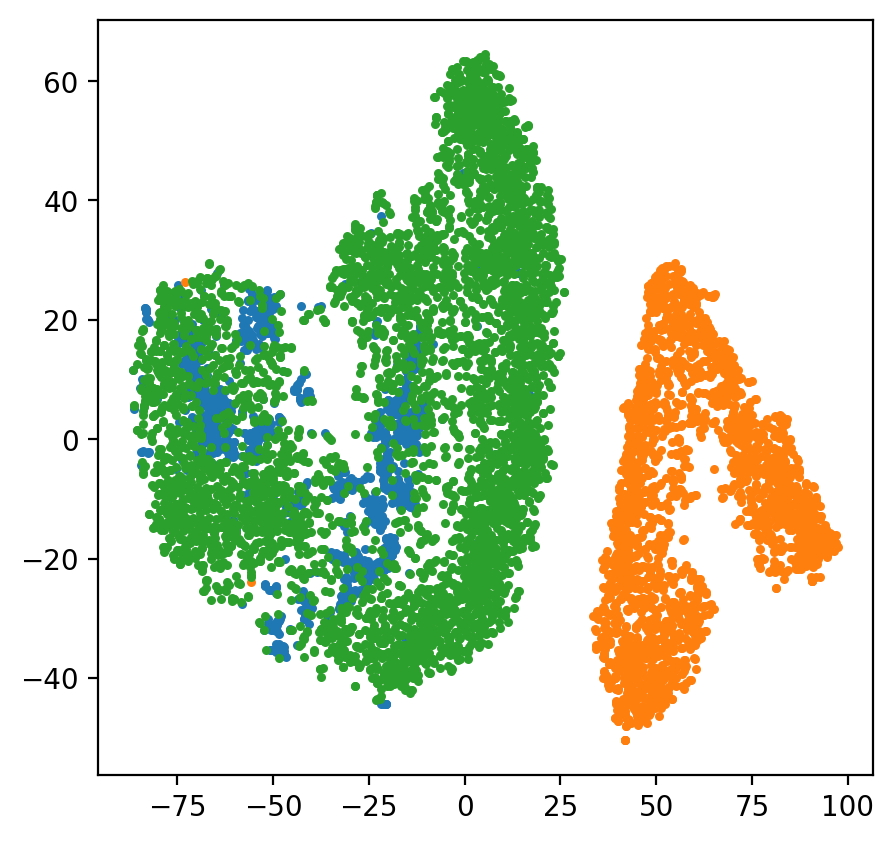}
    \caption{StarLightCurves.} \label{fig:tsne-star}
  \end{subfigure}%
  \begin{subfigure}{0.3\linewidth}
    \includegraphics[width=\linewidth]{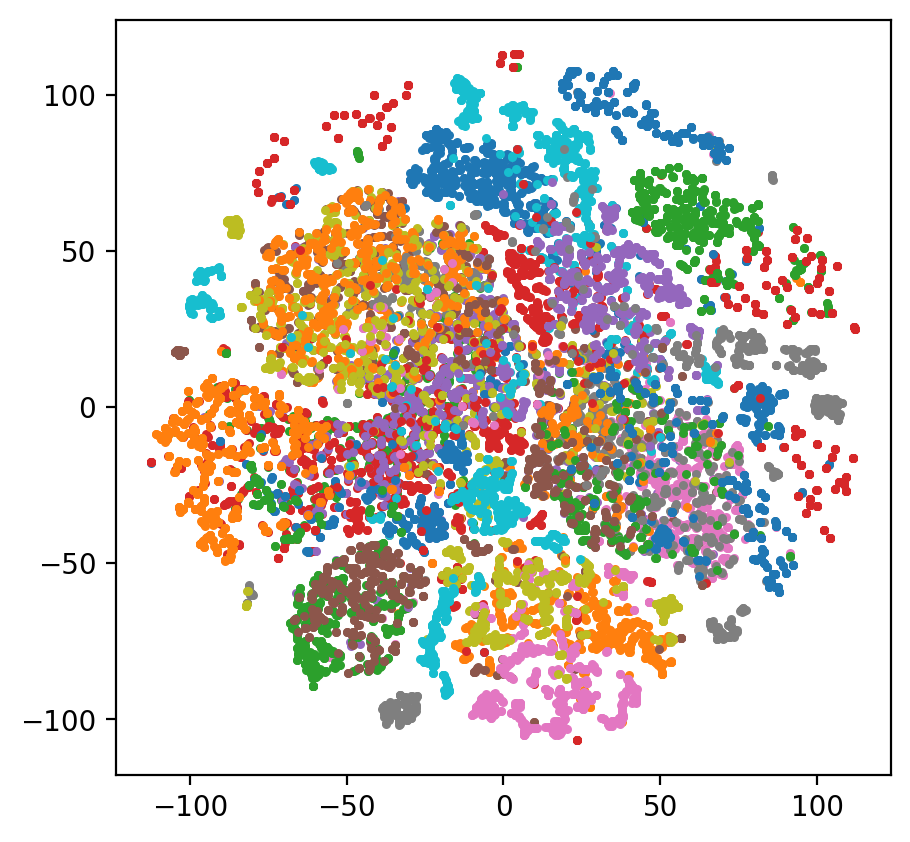}
    \caption{Crop.} \label{fig:tsne-crop}
  \end{subfigure}
  
  \caption{T-SNE visualizations of the learned representations of TS2Vec on the top 3 UCR datasets with the largest number of test samples. Different colors represent different classes.}\label{img:tsne}
\end{figure*}

\paragraph{Performance}
Table \ref{full-ucr-unsup} presents the full results of our method on 128 UCR datasets, compared with other existing methods of unsupervised representation, including T-Loss~\cite{TLoss}, TNC~\cite{TNC}, TS-TCC~\cite{TSTCC}, TST~\cite{TST} and DTW~\cite{chen2013dtw}. Among these baselines, TS2Vec achieves best accuracy on average. Besides, full results of TS2Vec for 30 multivariate datasets in the UEA archive are also provided in Table \ref{full-uea-unsup}, where TS2Vec provides best average performance.

\paragraph{Influence of the Batch Size}
The results of TS2Vec trained with different batch sizes ($B$) are also shown in Table \ref{full-ucr-unsup}. Although different $B$s get close average scores, there are notable differences between different $B$s on the scores of individual datasets.

\paragraph{Other Baselines}
To test the \textit{transferability} of the representations, for each UCR dataset, we use the representations computed by an encoder trained on another dataset \textit{FordA} following the setting in T-Loss~\cite{TLoss}. Table \ref{full-ucr-other} shows that the transfer version of our method (TS2Vec$^\dag$) achieves a 3.8\% average accuracy improvement compared to T-Loss's transfer version (T-Loss$^\dag$). Besides, the scores achieved by TS2Vec$^\dag$ are close to these of our non-transfer version, demonstrating the \textit{transferability} of TS2Vec from a dataset to another.

Note that T-Loss also provides an ensemble version, denoted as T-Loss-4X in this paper, which combines the learned representations from 4 encoders trained with different number of negative samples. Table \ref{full-ucr-other} shows that our non-ensemble version (with a representation size of $320$) outperforms T-Loss-4X (with a representation size of $1280$).

\paragraph{Visualization}
We visualize the learned time series representation by T-SNE in Figure \ref{img:tsne}. It implies that the learned representations can distinguish different classes in the latent space.

\begin{table*}[!ht]
  \centering
  \scalebox{0.9}{
  \begin{tabular}{lccccccccccccc}
  \toprule
         &  & \multicolumn{2}{c}{TS2Vec} &
         \multicolumn{2}{c}{Informer} &
         \multicolumn{2}{c}{LogTrans} &
         \multicolumn{2}{c}{N-BEATS} &
         \multicolumn{2}{c}{TCN} &
         \multicolumn{2}{c}{LSTnet} \\
    \cmidrule(r){3-4} \cmidrule(r){5-6} \cmidrule(r){7-8} \cmidrule(r){9-10} \cmidrule(r){11-12} \cmidrule(r){13-14}
    Dataset & H & MSE & MAE & MSE & MAE & MSE & MAE & MSE & MAE & MSE & MAE & MSE & MAE \\
    \midrule
    \multirow{5}*{ETTh$_1$}
    & 24 & \textbf{0.039} & \textbf{0.152} & 0.098 & 0.247 & 0.103 & 0.259 & 0.094 & 0.238 & 0.075 & 0.210 & 0.108 & 0.284 \\
    & 48 & \textbf{0.062} & \textbf{0.191} & 0.158 & 0.319 & 0.167 & 0.328 & 0.210 & 0.367 & 0.227 & 0.402 & 0.175 & 0.424 \\
    & 168 & \textbf{0.134} & \textbf{0.282} & 0.183 & 0.346 & 0.207 & 0.375 & 0.232 & 0.391 & 0.316 & 0.493 & 0.396 & 0.504 \\
    & 336 & \textbf{0.154} & \textbf{0.310} & 0.222 & 0.387 & 0.230 & 0.398 & 0.232 & 0.388 & 0.306 & 0.495 & 0.468 & 0.593 \\
    & 720 & \textbf{0.163} & \textbf{0.327} & 0.269 & 0.435 & 0.273 & 0.463 & 0.322 & 0.490 & 0.390 & 0.557 & 0.659 & 0.766 \\
    
    \midrule
    
    \multirow{5}*{ETTh$_2$}
    & 24 & \textbf{0.090} & \textbf{0.229} & 0.093 & 0.240 & 0.102 & 0.255 & 0.198 & 0.345 & 0.103 & 0.249 & 3.554 & 0.445 \\
    & 48 & \textbf{0.124} & \textbf{0.273} & 0.155 & 0.314 & 0.169 & 0.348 & 0.234 & 0.386 & 0.142 & 0.290 & 3.190 & 0.474 \\
    & 168 & \textbf{0.208} & \textbf{0.360} & 0.232 & 0.389 & 0.246 & 0.422 & 0.331 & 0.453 & 0.227 & 0.376 & 2.800 & 0.595 \\
    & 336 & \textbf{0.213} & \textbf{0.369} & 0.263 & 0.417 & 0.267 & 0.437 & 0.431 & 0.508 & 0.296 & 0.430 & 2.753 & 0.738 \\
    & 720 & \textbf{0.214} & \textbf{0.374} & 0.277 & 0.431 & 0.303 & 0.493 & 0.437 & 0.517 & 0.325 & 0.463 & 2.878 & 1.044 \\
    
    \midrule
    
    \multirow{5}*{ETTm$_1$}
    & 24 & \textbf{0.015} & \textbf{0.092} & 0.030 & 0.137 & 0.065 & 0.202 & 0.054 & 0.184 & 0.041 & 0.157 & 0.090 & 0.206   \\
    & 48 & \textbf{0.027} & \textbf{0.126} & 0.069 & 0.203 & 0.078 & 0.220 & 0.190 & 0.361 & 0.101 & 0.257 & 0.179 & 0.306 \\
    & 96 & \textbf{0.044} & \textbf{0.161} & 0.194 & 0.372 & 0.199 & 0.386 & 0.183 & 0.353 & 0.142 & 0.311 & 0.272 & 0.399 \\
    & 288 & \textbf{0.103} & \textbf{0.246} & 0.401 & 0.554 & 0.411 & 0.572 & 0.186 & 0.362 & 0.318 & 0.472 & 0.462 & 0.558 \\
    & 672 & \textbf{0.156} & \textbf{0.307} & 0.512 & 0.644 & 0.598 & 0.702 & 0.197 & 0.368 & 0.397 & 0.547 & 0.639 & 0.697 \\

    \midrule
    
    \multirow{5}*{Electricity}
    & 24 & 0.260 & 0.288 & \textbf{0.251} & \textbf{0.275} & 0.528 & 0.447 & 0.427 & 0.330 & 0.263 & 0.279 & 0.281 & 0.287 \\
    & 48 & \textbf{0.319} & \textbf{0.324} & 0.346 & 0.339 & 0.409 & 0.414 & 0.551 & 0.392 & 0.373 & 0.344 & 0.381 & 0.366 \\
    & 168 & \textbf{0.427} & \textbf{0.394} & 0.544 & 0.424 & 0.959 & 0.612 & 0.893 & 0.538 & 0.609 & 0.462 & 0.599 & 0.500 \\
    & 336 & \textbf{0.565} & \textbf{0.474} & 0.713 & 0.512 & 1.079 & 0.639 & 1.035 & 0.669 & 0.855 & 0.606 & 0.823 & 0.624 \\
    & 720 & \textbf{0.861} & \textbf{0.643} & 1.182 & 0.806 & 1.001 & 0.714 & 1.548 & 0.881 & 1.263 & 0.858 & 1.278 & 0.906 \\

    \midrule
    
    \multicolumn{2}{l}{Avg.} & \textbf{0.209} & \textbf{0.296} & 0.310 & 0.390 & 0.370 & 0.434 & 0.399 & 0.426 & 0.338 & 0.413 & 1.099 & 0.536 \\
    
    \bottomrule
  \end{tabular}
  }
  \caption{Univariate time series forecasting results.}
  \label{forecast-univar}
\end{table*}

\begin{table*}
  \centering
  \scalebox{0.9}{
  \begin{threeparttable}
  \begin{tabular}{lccccccccccccc}
  \toprule
         &  & \multicolumn{2}{c}{TS2Vec} &
         \multicolumn{2}{c}{Informer} &
         \multicolumn{2}{c}{StemGNN} &
         \multicolumn{2}{c}{TCN} &
         \multicolumn{2}{c}{LogTrans} &
         \multicolumn{2}{c}{LSTnet} \\
    \cmidrule(r){3-4} \cmidrule(r){5-6} \cmidrule(r){7-8} \cmidrule(r){9-10} \cmidrule(r){11-12} \cmidrule(r){13-14} 
    Dataset & H & MSE & MAE & MSE & MAE & MSE & MAE & MSE & MAE & MSE & MAE & MSE & MAE \\
    \midrule
    \multirow{5}*{ETTh$_1$}
    & 24 & 0.599 & \textbf{0.534} & \textbf{0.577} & 0.549 & 0.614 & 0.571 & 0.767 & 0.612 & 0.686 & 0.604 & 1.293 & 0.901 \\
    & 48 & \textbf{0.629} & \textbf{0.555} & 0.685 & 0.625 & 0.748 & 0.618 & 0.713 & 0.617 & 0.766 & 0.757 & 1.456 & 0.960 \\
    & 168 & 0.755 & 0.636 & 0.931 & 0.752 & \textbf{0.663} & \textbf{0.608} & 0.995 & 0.738 & 1.002 & 0.846 & 1.997 & 1.214 \\
    & 336 & \textbf{0.907} & \textbf{0.717} & 1.128 & 0.873 & 0.927 & 0.730 & 1.175 & 0.800 & 1.362 & 0.952 & 2.655 & 1.369 \\
    & 720 & \textbf{1.048} & \textbf{0.790} & 1.215 & 0.896 & --\tnote{*} & -- & 1.453 & 1.311 & 1.397 & 1.291 & 2.143 & 1.380 \\
    
    \midrule

    \multirow{5}*{ETTh$_2$}
    & 24 & \textbf{0.398} & \textbf{0.461} & 0.720 & 0.665 & 1.292 & 0.883 & 1.365 & 0.888 & 0.828 & 0.750 & 2.742 & 1.457 \\
    & 48 & \textbf{0.580} & \textbf{0.573} & 1.457 & 1.001 & 1.099 & 0.847 & 1.395 & 0.960 & 1.806 & 1.034 & 3.567 & 1.687 \\
    & 168 & \textbf{1.901} & \textbf{1.065} & 3.489 & 1.515 & 2.282 & 1.228 & 3.166 & 1.407 & 4.070 & 1.681 & 3.242 & 2.513 \\
    & 336 & \textbf{2.304} & \textbf{1.215} & 2.723 & 1.340 & 3.086 & 1.351 & 3.256 & 1.481 & 3.875 & 1.763 & 2.544 & 2.591 \\
    & 720 & \textbf{2.650} & \textbf{1.373} & 3.467 & 1.473 & -- & -- & 3.690 & 1.588 & 3.913 & 1.552 & 4.625 & 3.709 \\
    
    \midrule
    
    \multirow{5}*{ETTm$_1$}
    & 24 & 0.443 & 0.436 & \textbf{0.323} & \textbf{0.369} & 0.620 & 0.570 & 0.324 & 0.374 & 0.419 & 0.412 & 1.968 & 1.170 \\
    & 48 & 0.582 & 0.515 & 0.494 & 0.503 & 0.744 & 0.628 & \textbf{0.477} & \textbf{0.450} & 0.507 & 0.583 & 1.999 & 1.215 \\
    & 96 & \textbf{0.622} & \textbf{0.549} & 0.678 & 0.614 & 0.709 & 0.624 & 0.636 & 0.602 & 0.768 & 0.792 & 2.762 & 1.542 \\
    & 288 & \textbf{0.709} & \textbf{0.609} & 1.056 & 0.786 & 0.843 & 0.683 & 1.270 & 1.351 & 1.462 & 1.320 & 1.257 & 2.076 \\
    & 672 & \textbf{0.786} & \textbf{0.655} & 1.192 & 0.926 & -- & -- & 1.381 & 1.467 & 1.669 & 1.461 & 1.917 & 2.941 \\

    \midrule
    
    \multirow{5}*{Electricity}
    & 24 & \textbf{0.287} & \textbf{0.374} & 0.312 & 0.387 & 0.439 & 0.388 & 0.305 & 0.384 & 0.297 & \textbf{0.374} & 0.356 & 0.419 \\
    & 48 & \textbf{0.307} & \textbf{0.388} & 0.392 & 0.431 & 0.413 & 0.455 & 0.317 & 0.392 & 0.316 & 0.389 & 0.429 & 0.456 \\
    & 168 & \textbf{0.332} & \textbf{0.407} & 0.515 & 0.509 & 0.506 & 0.518 & 0.358 & 0.423 & 0.426 & 0.466 & 0.372 & 0.425 \\
    & 336 & \textbf{0.349} & 0.420 & 0.759 & 0.625 & 0.647 & 0.596 & \textbf{0.349} & 0.416 & 0.365 & 0.417 & 0.352 & \textbf{0.409} \\
    & 720 & 0.375 & 0.438 & 0.969 & 0.788 & -- & -- & 0.447 & 0.486 & \textbf{0.344} & \textbf{0.403} & 0.380 & 0.443 \\
    
    \midrule
    
    \multicolumn{2}{l}{Avg.} & \textbf{0.828} & \textbf{0.636} & 1.154 & 0.781 & -- & -- & 1.192 & 0.837 & 1.314 & 0.892 & 1.903 & 1.444 \\
    \bottomrule
  \end{tabular}
  \begin{tablenotes}
    \footnotesize
        \item[*] All $H \geq 672$ cases of StemGNN fail for the out-of-memory (24GB) even when $batch\_size=1$.
  \end{tablenotes}
  \end{threeparttable}
 }
  \caption{Multivariate time series forecasting results.}
  \label{forecast-multivar}
\end{table*}

\clearpage

\onecolumn
\relsize{-1}
\begin{longtable}{lcccccccc}
  \toprule
     & \multicolumn{3}{c}{TS2Vec} &
     \multirow{2}*{T-Loss} &
     \multirow{2}*{TNC} &
     \multirow{2}*{TS-TCC} &
     \multirow{2}*{TST} &
     \multirow{2}*{DTW}\\
    \cmidrule(r){2-4}
    Dataset & B=4 & B=8 & B=16 \\
    \midrule
    \endhead
Adiac & \textbf{0.775} & 0.762 & 0.765 & 0.675 & 0.726 & 0.767 & 0.550 & 0.604 \\
ArrowHead & 0.823 & \textbf{0.857} & 0.817 & 0.766 & 0.703 & 0.737 & 0.771 & 0.703 \\
Beef & \textbf{0.767} & \textbf{0.767} & 0.633 & 0.667 & 0.733 & 0.600 & 0.500 & 0.633 \\
BeetleFly & 0.850 & 0.900 & 0.900 & 0.800 & 0.850 & 0.800 & \textbf{1.000} & 0.700 \\
BirdChicken & 0.800 & 0.800 & 0.800 & \textbf{0.850} & 0.750 & 0.650 & 0.650 & 0.750 \\
Car & \textbf{0.883} & 0.833 & 0.700 & 0.833 & 0.683 & 0.583 & 0.550 & 0.733 \\
CBF & \textbf{1.000} & \textbf{1.000} & \textbf{1.000} & 0.983 & 0.983 & 0.998 & 0.898 & 0.997 \\
ChlorineConcentration & 0.810 & \textbf{0.832} & 0.812 & 0.749 & 0.760 & 0.753 & 0.562 & 0.648 \\
CinCECGTorso & 0.812 & \textbf{0.827} & 0.825 & 0.713 & 0.669 & 0.671 & 0.508 & 0.651 \\
Coffee & \textbf{1.000} & \textbf{1.000} & \textbf{1.000} & \textbf{1.000} & \textbf{1.000} & \textbf{1.000} & 0.821 & \textbf{1.000} \\
Computers & 0.636 & 0.660 & 0.660 & 0.664 & 0.684 & \textbf{0.704} & 0.696 & 0.700 \\
CricketX & 0.800 & 0.782 & \textbf{0.805} & 0.713 & 0.623 & 0.731 & 0.385 & 0.754 \\
CricketY & 0.756 & 0.749 & \textbf{0.769} & 0.728 & 0.597 & 0.718 & 0.467 & 0.744 \\
CricketZ & 0.785 & \textbf{0.792} & 0.790 & 0.708 & 0.682 & 0.713 & 0.403 & 0.754 \\
DiatomSizeReduction & 0.980 & 0.984 & 0.987 & 0.984 & \textbf{0.993} & 0.977 & 0.961 & 0.967 \\
DistalPhalanxOutlineCorrect & \textbf{0.775} & 0.761 & 0.757 & \textbf{0.775} & 0.754 & 0.754 & 0.728 & 0.717 \\
DistalPhalanxOutlineAgeGroup & 0.719 & 0.727 & 0.719 & 0.727 & 0.741 & 0.755 & 0.741 & \textbf{0.770} \\
DistalPhalanxTW & 0.662 & \textbf{0.698} & 0.683 & 0.676 & 0.669 & 0.676 & 0.568 & 0.590 \\
Earthquakes & \textbf{0.748} & \textbf{0.748} & \textbf{0.748} & \textbf{0.748} & \textbf{0.748} & \textbf{0.748} & \textbf{0.748} & 0.719 \\
ECG200 & 0.890 & 0.920 & 0.880 & \textbf{0.940} & 0.830 & 0.880 & 0.830 & 0.770 \\
ECG5000 & 0.935 & 0.935 & 0.934 & 0.933 & 0.937 & \textbf{0.941} & 0.928 & 0.924 \\
ECGFiveDays & \textbf{1.000} & \textbf{1.000} & \textbf{1.000} & \textbf{1.000} & 0.999 & 0.878 & 0.763 & 0.768 \\
ElectricDevices & 0.712 & \textbf{0.721} & 0.719 & 0.707 & 0.700 & 0.686 & 0.676 & 0.602 \\
FaceAll & 0.759 & 0.771 & 0.805 & 0.786 & 0.766 & \textbf{0.813} & 0.504 & 0.808 \\
FaceFour & 0.864 & \textbf{0.932} & \textbf{0.932} & 0.920 & 0.659 & 0.773 & 0.511 & 0.830 \\
FacesUCR & \textbf{0.930} & 0.924 & 0.926 & 0.884 & 0.789 & 0.863 & 0.543 & 0.905 \\
FiftyWords & 0.771 & 0.771 & \textbf{0.774} & 0.732 & 0.653 & 0.653 & 0.525 & 0.690 \\
Fish & \textbf{0.937} & 0.926 & \textbf{0.937} & 0.891 & 0.817 & 0.817 & 0.720 & 0.823 \\
FordA & 0.940 & 0.936 & \textbf{0.948} & 0.928 & 0.902 & 0.930 & 0.568 & 0.555 \\
FordB & 0.789 & 0.794 & 0.807 & 0.793 & 0.733 & \textbf{0.815} & 0.507 & 0.620 \\
GunPoint & 0.980 & 0.980 & 0.987 & 0.980 & 0.967 & \textbf{0.993} & 0.827 & 0.907 \\
Ham & 0.714 & 0.714 & 0.724 & 0.724 & \textbf{0.752} & 0.743 & 0.524 & 0.467 \\
HandOutlines & 0.919 & 0.922 & \textbf{0.930} & 0.922 & \textbf{0.930} & 0.724 & 0.735 & 0.881 \\
Haptics & 0.510 & 0.526 & \textbf{0.536} & 0.490 & 0.474 & 0.396 & 0.357 & 0.377 \\
Herring & 0.625 & \textbf{0.641} & 0.609 & 0.594 & 0.594 & 0.594 & 0.594 & 0.531 \\
InlineSkate & 0.389 & \textbf{0.415} & 0.407 & 0.371 & 0.378 & 0.347 & 0.287 & 0.384 \\
InsectWingbeatSound & 0.629 & \textbf{0.630} & 0.624 & 0.597 & 0.549 & 0.415 & 0.266 & 0.355 \\
ItalyPowerDemand & \textbf{0.961} & 0.925 & 0.960 & 0.954 & 0.928 & 0.955 & 0.845 & 0.950 \\
LargeKitchenAppliances & 0.845 & 0.845 & \textbf{0.875} & 0.789 & 0.776 & 0.848 & 0.595 & 0.795 \\
Lightning2 & 0.836 & \textbf{0.869} & 0.820 & \textbf{0.869} & \textbf{0.869} & 0.836 & 0.705 & \textbf{0.869} \\
Lightning7 & 0.836 & \textbf{0.863} & 0.822 & 0.795 & 0.767 & 0.685 & 0.411 & 0.726 \\
Mallat & 0.915 & 0.914 & 0.873 & \textbf{0.951} & 0.871 & 0.922 & 0.713 & 0.934 \\
Meat & 0.950 & 0.950 & \textbf{0.967} & 0.950 & 0.917 & 0.883 & 0.900 & 0.933 \\
MedicalImages & 0.792 & 0.789 & \textbf{0.793} & 0.750 & 0.754 & 0.747 & 0.632 & 0.737 \\
MiddlePhalanxOutlineCorrect & 0.811 & \textbf{0.838} & 0.825 & 0.825 & 0.818 & 0.818 & 0.753 & 0.698 \\
MiddlePhalanxOutlineAgeGroup & 0.636 & 0.636 & 0.630 & \textbf{0.656} & 0.643 & 0.630 & 0.617 & 0.500 \\
MiddlePhalanxTW & 0.591 & 0.584 & 0.578 & 0.591 & 0.571 & \textbf{0.610} & 0.506 & 0.506 \\
MoteStrain & 0.857 & 0.861 & \textbf{0.863} & 0.851 & 0.825 & 0.843 & 0.768 & 0.835 \\
NonInvasiveFetalECGThorax1 & 0.923 & \textbf{0.930} & 0.919 & 0.878 & 0.898 & 0.898 & 0.471 & 0.790 \\
NonInvasiveFetalECGThorax2 & \textbf{0.940} & 0.938 & 0.935 & 0.919 & 0.912 & 0.913 & 0.832 & 0.865 \\
OliveOil & \textbf{0.900} & \textbf{0.900} & \textbf{0.900} & 0.867 & 0.833 & 0.800 & 0.800 & 0.833 \\
OSULeaf & \textbf{0.876} & 0.851 & 0.843 & 0.760 & 0.723 & 0.723 & 0.545 & 0.591 \\
PhalangesOutlinesCorrect & 0.795 & 0.809 & \textbf{0.823} & 0.784 & 0.787 & 0.804 & 0.773 & 0.728 \\
Phoneme & 0.296 & \textbf{0.312} & 0.309 & 0.276 & 0.180 & 0.242 & 0.139 & 0.228 \\
Plane & \textbf{1.000} & \textbf{1.000} & 0.990 & 0.990 & \textbf{1.000} & \textbf{1.000} & 0.933 & \textbf{1.000} \\
ProximalPhalanxOutlineCorrect & 0.876 & 0.887 & \textbf{0.900} & 0.859 & 0.866 & 0.873 & 0.770 & 0.784 \\
ProximalPhalanxOutlineAgeGroup & 0.844 & 0.834 & 0.829 & 0.844 & \textbf{0.854} & 0.839 & \textbf{0.854} & 0.805 \\
ProximalPhalanxTW & 0.785 & \textbf{0.824} & 0.805 & 0.771 & 0.810 & 0.800 & 0.780 & 0.761 \\
RefrigerationDevices & 0.587 & \textbf{0.589} & \textbf{0.589} & 0.515 & 0.565 & 0.563 & 0.483 & 0.464 \\
ScreenType & 0.405 & 0.411 & 0.397 & 0.416 & \textbf{0.509} & 0.419 & 0.419 & 0.397 \\
ShapeletSim & 0.989 & \textbf{1.000} & 0.994 & 0.672 & 0.589 & 0.683 & 0.489 & 0.650 \\
ShapesAll & 0.897 & 0.902 & \textbf{0.905} & 0.848 & 0.788 & 0.773 & 0.733 & 0.768 \\
SmallKitchenAppliances & 0.723 & 0.731 & \textbf{0.733} & 0.677 & 0.725 & 0.691 & 0.592 & 0.643 \\
SonyAIBORobotSurface1 & 0.874 & \textbf{0.903} & 0.900 & 0.902 & 0.804 & 0.899 & 0.724 & 0.725 \\
SonyAIBORobotSurface2 & 0.890 & 0.871 & 0.889 & 0.889 & 0.834 & \textbf{0.907} & 0.745 & 0.831 \\
StarLightCurves & 0.970 & 0.969 & \textbf{0.971} & 0.964 & 0.968 & 0.967 & 0.949 & 0.907 \\
Strawberry & 0.962 & 0.962 & \textbf{0.965} & 0.954 & 0.951 & \textbf{0.965} & 0.916 & 0.941 \\
SwedishLeaf & 0.939 & 0.941 & \textbf{0.942} & 0.914 & 0.880 & 0.923 & 0.738 & 0.792 \\
Symbols & 0.973 & \textbf{0.976} & 0.972 & 0.963 & 0.885 & 0.916 & 0.786 & 0.950 \\
SyntheticControl & 0.997 & 0.997 & 0.993 & 0.987 & \textbf{1.000} & 0.990 & 0.490 & 0.993 \\
ToeSegmentation1 & 0.930 & 0.917 & \textbf{0.947} & 0.939 & 0.864 & 0.930 & 0.807 & 0.772 \\
ToeSegmentation2 & \textbf{0.915} & 0.892 & 0.900 & 0.900 & 0.831 & 0.877 & 0.615 & 0.838 \\
Trace & \textbf{1.000} & \textbf{1.000} & \textbf{1.000} & 0.990 & \textbf{1.000} & \textbf{1.000} & \textbf{1.000} & \textbf{1.000} \\
TwoLeadECG & 0.982 & 0.986 & 0.987 & \textbf{0.999} & 0.993 & 0.976 & 0.871 & 0.905 \\
TwoPatterns & \textbf{1.000} & \textbf{1.000} & \textbf{1.000} & 0.999 & \textbf{1.000} & 0.999 & 0.466 & \textbf{1.000} \\
UWaveGestureLibraryX & \textbf{0.810} & 0.795 & 0.801 & 0.785 & 0.781 & 0.733 & 0.569 & 0.728 \\
UWaveGestureLibraryY & \textbf{0.729} & 0.719 & 0.720 & 0.710 & 0.697 & 0.641 & 0.348 & 0.634 \\
UWaveGestureLibraryZ & 0.761 & \textbf{0.770} & 0.768 & 0.757 & 0.721 & 0.690 & 0.655 & 0.658 \\
UWaveGestureLibraryAll & \textbf{0.934} & 0.930 & \textbf{0.934} & 0.896 & 0.903 & 0.692 & 0.475 & 0.892 \\
Wafer & 0.995 & \textbf{0.998} & 0.997 & 0.992 & 0.994 & 0.994 & 0.991 & 0.980 \\
Wine & 0.778 & 0.870 & \textbf{0.889} & 0.815 & 0.759 & 0.778 & 0.500 & 0.574 \\
WordSynonyms & 0.699 & 0.676 & \textbf{0.704} & 0.691 & 0.630 & 0.531 & 0.422 & 0.649 \\
Worms & 0.701 & 0.701 & 0.701 & 0.727 & 0.623 & \textbf{0.753} & 0.455 & 0.584 \\
WormsTwoClass & \textbf{0.805} & \textbf{0.805} & 0.753 & 0.792 & 0.727 & 0.753 & 0.584 & 0.623 \\
Yoga & 0.880 & \textbf{0.887} & 0.877 & 0.837 & 0.812 & 0.791 & 0.830 & 0.837 \\
ACSF1 & 0.840 & 0.900 & \textbf{0.910} & 0.900 & 0.730 & 0.730 & 0.760 & 0.640 \\
AllGestureWiimoteX & 0.744 & \textbf{0.777} & 0.751 & 0.763 & 0.703 & 0.697 & 0.259 & 0.716 \\
AllGestureWiimoteY & 0.764 & \textbf{0.793} & 0.774 & 0.726 & 0.699 & 0.741 & 0.423 & 0.729 \\
AllGestureWiimoteZ & 0.734 & 0.746 & \textbf{0.770} & 0.723 & 0.646 & 0.689 & 0.447 & 0.643 \\
BME & 0.973 & \textbf{0.993} & 0.980 & \textbf{0.993} & 0.973 & 0.933 & 0.760 & 0.900 \\
Chinatown & 0.968 & 0.965 & 0.959 & 0.951 & 0.977 & \textbf{0.983} & 0.936 & 0.957 \\
Crop & 0.753 & \textbf{0.756} & 0.753 & 0.722 & 0.738 & 0.742 & 0.710 & 0.665 \\
EOGHorizontalSignal & 0.544 & 0.539 & 0.522 & \textbf{0.605} & 0.442 & 0.401 & 0.373 & 0.503 \\
EOGVerticalSignal & 0.467 & \textbf{0.503} & 0.472 & 0.434 & 0.392 & 0.376 & 0.298 & 0.448 \\
EthanolLevel & 0.480 & 0.468 & 0.484 & 0.382 & 0.424 & \textbf{0.486} & 0.260 & 0.276 \\
FreezerRegularTrain & 0.985 & 0.986 & 0.983 & 0.956 & \textbf{0.991} & 0.989 & 0.922 & 0.899 \\
FreezerSmallTrain & 0.894 & 0.870 & 0.872 & 0.933 & \textbf{0.982} & 0.979 & 0.920 & 0.753 \\
Fungi & 0.962 & 0.957 & 0.946 & \textbf{1.000} & 0.527 & 0.753 & 0.366 & 0.839 \\
GestureMidAirD1 & \textbf{0.631} & 0.608 & 0.615 & 0.608 & 0.431 & 0.369 & 0.208 & 0.569 \\
GestureMidAirD2 & 0.508 & 0.469 & 0.515 & 0.546 & 0.362 & 0.254 & 0.138 & \textbf{0.608} \\
GestureMidAirD3 & \textbf{0.346} & 0.292 & 0.300 & 0.285 & 0.292 & 0.177 & 0.154 & 0.323 \\
GesturePebbleZ1 & 0.878 & \textbf{0.930} & 0.884 & 0.919 & 0.378 & 0.395 & 0.500 & 0.791 \\
GesturePebbleZ2 & 0.842 & 0.873 & 0.848 & \textbf{0.899} & 0.316 & 0.430 & 0.380 & 0.671 \\
GunPointAgeSpan & \textbf{0.994} & 0.987 & 0.968 & \textbf{0.994} & 0.984 & \textbf{0.994} & 0.991 & 0.918 \\
GunPointMaleVersusFemale & \textbf{1.000} & \textbf{1.000} & \textbf{1.000} & 0.997 & 0.994 & 0.997 & \textbf{1.000} & 0.997 \\
GunPointOldVersusYoung & \textbf{1.000} & \textbf{1.000} & \textbf{1.000} & \textbf{1.000} & \textbf{1.000} & \textbf{1.000} & \textbf{1.000} & 0.838 \\
HouseTwenty & \textbf{0.941} & 0.916 & \textbf{0.941} & 0.933 & 0.782 & 0.790 & 0.815 & 0.924 \\
InsectEPGRegularTrain & \textbf{1.000} & \textbf{1.000} & \textbf{1.000} & \textbf{1.000} & \textbf{1.000} & \textbf{1.000} & \textbf{1.000} & 0.872 \\
InsectEPGSmallTrain & \textbf{1.000} & \textbf{1.000} & \textbf{1.000} & \textbf{1.000} & \textbf{1.000} & \textbf{1.000} & \textbf{1.000} & 0.735 \\
MelbournePedestrian & 0.954 & \textbf{0.959} & 0.956 & 0.944 & 0.942 & 0.949 & 0.741 & 0.791 \\
MixedShapesRegularTrain & 0.915 & 0.917 & \textbf{0.922} & 0.905 & 0.911 & 0.855 & 0.879 & 0.842 \\
MixedShapesSmallTrain & \textbf{0.881} & 0.861 & 0.856 & 0.860 & 0.813 & 0.735 & 0.828 & 0.780 \\
PickupGestureWiimoteZ & 0.800 & \textbf{0.820} & 0.760 & 0.740 & 0.620 & 0.600 & 0.240 & 0.660 \\
PigAirwayPressure & 0.524 & 0.630 & \textbf{0.683} & 0.510 & 0.413 & 0.380 & 0.120 & 0.106 \\
PigArtPressure & 0.962 & \textbf{0.966} & \textbf{0.966} & 0.928 & 0.808 & 0.524 & 0.774 & 0.245 \\
PigCVP & 0.803 & 0.812 & \textbf{0.870} & 0.788 & 0.649 & 0.615 & 0.596 & 0.154 \\
PLAID & 0.551 & 0.561 & 0.549 & 0.555 & 0.495 & 0.445 & 0.419 & \textbf{0.840} \\
PowerCons & 0.967 & 0.961 & \textbf{0.972} & 0.900 & 0.933 & 0.961 & 0.911 & 0.878 \\
Rock & 0.660 & \textbf{0.700} & \textbf{0.700} & 0.580 & 0.580 & 0.600 & 0.680 & 0.600 \\
SemgHandGenderCh2 & 0.952 & \textbf{0.963} & 0.962 & 0.890 & 0.882 & 0.837 & 0.725 & 0.802 \\
SemgHandMovementCh2 & \textbf{0.893} & 0.860 & 0.891 & 0.789 & 0.593 & 0.613 & 0.420 & 0.584 \\
SemgHandSubjectCh2 & 0.944 & \textbf{0.951} & 0.942 & 0.853 & 0.771 & 0.753 & 0.484 & 0.727 \\
ShakeGestureWiimoteZ & \textbf{0.940} & \textbf{0.940} & 0.920 & 0.920 & 0.820 & 0.860 & 0.760 & 0.860 \\
SmoothSubspace & 0.967 & 0.980 & \textbf{0.993} & 0.960 & 0.913 & 0.953 & 0.827 & 0.827 \\
UMD & \textbf{1.000} & \textbf{1.000} & 0.993 & 0.993 & 0.993 & 0.986 & 0.910 & 0.993 \\
DodgerLoopDay & 0.425 & \textbf{0.562} & 0.500 & -- & -- & -- & 0.200 & 0.500 \\
DodgerLoopGame & 0.826 & 0.841 & 0.819 & -- & -- & -- & 0.696 & \textbf{0.877} \\
DodgerLoopWeekend & 0.942 & \textbf{0.964} & 0.942 & -- & -- & -- & 0.732 & 0.949 \\
    \midrule
    On the first 125 datasets: \\
    AVG & 0.824 & \textbf{0.830} & 0.827 & 0.806 & 0.761 & 0.757 & 0.641 & 0.727 \\
    Rank & 3.048 & \textbf{2.688} & 2.820 & 4.160 & 5.136 & 4.936 & 7.060 & 6.152 \\
    \bottomrule
  \caption{Accuracy scores of our method compared with those of other methods of unsupervised representation on 128 UCR datasets. The representation dimensions of TS2Vec, T-Loss, TNC, TS-TCC and TST are all set to 320 for fair comparison.}\label{full-ucr-unsup}
\end{longtable}

\normalsize

\relsize{-1}
\begin{longtable}[t]{lccccccccc}
  \toprule
    Dataset & TS2Vec & T-Loss & TNC & TS-TCC & TST & DTW \\
    \midrule
    \endhead
ArticularyWordRecognition & \textbf{0.987} & 0.943 & 0.973 & 0.953 & 0.977 & \textbf{0.987} \\
AtrialFibrillation & 0.200 & 0.133 & 0.133 & \textbf{0.267} & 0.067 & 0.200 \\
BasicMotions & 0.975 & \textbf{1.000} & 0.975 & \textbf{1.000} & 0.975 & 0.975 \\
CharacterTrajectories & \textbf{0.995} & 0.993 & 0.967 & 0.985 & 0.975 & 0.989 \\
Cricket & 0.972 & 0.972 & 0.958 & 0.917 & \textbf{1.000} & \textbf{1.000} \\
DuckDuckGeese & \textbf{0.680} & 0.650 & 0.460 & 0.380 & 0.620 & 0.600 \\
EigenWorms & \textbf{0.847} & 0.840 & 0.840 & 0.779 & 0.748 & 0.618 \\
Epilepsy & 0.964 & \textbf{0.971} & 0.957 & 0.957 & 0.949 & 0.964 \\
ERing & 0.874 & 0.133 & 0.852 & \textbf{0.904} & 0.874 & 0.133 \\
EthanolConcentration & 0.308 & 0.205 & 0.297 & 0.285 & 0.262 & \textbf{0.323} \\
FaceDetection & 0.501 & 0.513 & 0.536 & \textbf{0.544} & 0.534 & 0.529 \\
FingerMovements & 0.480 & \textbf{0.580} & 0.470 & 0.460 & 0.560 & 0.530 \\
HandMovementDirection & 0.338 & \textbf{0.351} & 0.324 & 0.243 & 0.243 & 0.231 \\
Handwriting & \textbf{0.515} & 0.451 & 0.249 & 0.498 & 0.225 & 0.286 \\
Heartbeat & 0.683 & 0.741 & 0.746 & \textbf{0.751} & 0.746 & 0.717 \\
JapaneseVowels & 0.984 & \textbf{0.989} & 0.978 & 0.930 & 0.978 & 0.949 \\
Libras & 0.867 & \textbf{0.883} & 0.817 & 0.822 & 0.656 & 0.870 \\
LSST & 0.537 & 0.509 & \textbf{0.595} & 0.474 & 0.408 & 0.551 \\
MotorImagery & 0.510 & 0.580 & 0.500 & \textbf{0.610} & 0.500 & 0.500 \\
NATOPS & \textbf{0.928} & 0.917 & 0.911 & 0.822 & 0.850 & 0.883 \\
PEMS-SF & 0.682 & 0.676 & 0.699 & 0.734 & \textbf{0.740} & 0.711 \\
PenDigits & \textbf{0.989} & 0.981 & 0.979 & 0.974 & 0.560 & 0.977 \\
PhonemeSpectra & 0.233 & 0.222 & 0.207 & \textbf{0.252} & 0.085 & 0.151 \\
RacketSports & \textbf{0.855} & \textbf{0.855} & 0.776 & 0.816 & 0.809 & 0.803 \\
SelfRegulationSCP1 & 0.812 & \textbf{0.843} & 0.799 & 0.823 & 0.754 & 0.775 \\
SelfRegulationSCP2 & \textbf{0.578} & 0.539 & 0.550 & 0.533 & 0.550 & 0.539 \\
SpokenArabicDigits & \textbf{0.988} & 0.905 & 0.934 & 0.970 & 0.923 & 0.963 \\
StandWalkJump & \textbf{0.467} & 0.333 & 0.400 & 0.333 & 0.267 & 0.200 \\
UWaveGestureLibrary & \textbf{0.906} & 0.875 & 0.759 & 0.753 & 0.575 & 0.903 \\
InsectWingbeat & 0.466 & 0.156 & \textbf{0.469} & 0.264 & 0.105 & -- \\
    \midrule
    On the first 29 datasets: \\
    AVG & \textbf{0.712} & 0.675 & 0.677 & 0.682 & 0.635 & 0.650 \\
    Rank & \textbf{2.397} & 3.121 & 3.845 & 3.534 & 4.362 & 3.741 \\
    \bottomrule
  \caption{Accuracy scores of our method compared with those of other methods of unsupervised representation on 30 UEA datasets. The representation dimensions of TS2Vec, T-Loss, TNC, TS-TCC and TST are all set to 320 for fair comparison.}\label{full-uea-unsup}
\end{longtable}

\normalsize

\clearpage

\relsize{-1}
\begin{longtable}[t]{lcccc}
  \toprule
    Dataset & TS2Vec & TS2Vec$^\dag$ & T-Loss$^\dag$ & T-Loss-4X \\
    \midrule
    \endhead
Adiac & 0.762 & \textbf{0.783} & 0.760 & 0.716 \\
ArrowHead & \textbf{0.857} & 0.829 & 0.817 & 0.829 \\
Beef & \textbf{0.767} & 0.700 & 0.667 & 0.700 \\
BeetleFly & \textbf{0.900} & \textbf{0.900} & 0.800 & \textbf{0.900} \\
BirdChicken & 0.800 & 0.800 & \textbf{0.900} & 0.800 \\
Car & 0.833 & 0.817 & \textbf{0.850} & 0.817 \\
CBF & \textbf{1.000} & \textbf{1.000} & 0.988 & 0.994 \\
ChlorineConcentration & \textbf{0.832} & 0.802 & 0.688 & 0.782 \\
CinCECGTorso & \textbf{0.827} & 0.738 & 0.638 & 0.740 \\
Coffee & \textbf{1.000} & \textbf{1.000} & \textbf{1.000} & \textbf{1.000} \\
Computers & \textbf{0.660} & \textbf{0.660} & 0.648 & 0.628 \\
CricketX & \textbf{0.782} & 0.767 & 0.682 & 0.777 \\
CricketY & 0.749 & 0.746 & 0.667 & \textbf{0.767} \\
CricketZ & \textbf{0.792} & 0.772 & 0.656 & 0.764 \\
DiatomSizeReduction & 0.984 & 0.961 & 0.974 & \textbf{0.993} \\
DistalPhalanxOutlineCorrect & 0.761 & 0.757 & 0.764 & \textbf{0.768} \\
DistalPhalanxOutlineAgeGroup & 0.727 & \textbf{0.748} & 0.727 & 0.734 \\
DistalPhalanxTW & \textbf{0.698} & 0.669 & 0.669 & 0.676 \\
Earthquakes & \textbf{0.748} & \textbf{0.748} & \textbf{0.748} & \textbf{0.748} \\
ECG200 & \textbf{0.920} & 0.910 & 0.830 & 0.900 \\
ECG5000 & 0.935 & 0.935 & \textbf{0.940} & 0.936 \\
ECGFiveDays & \textbf{1.000} & \textbf{1.000} & \textbf{1.000} & \textbf{1.000} \\
ElectricDevices & 0.721 & 0.714 & 0.676 & \textbf{0.732} \\
FaceAll & 0.771 & 0.786 & 0.734 & \textbf{0.802} \\
FaceFour & \textbf{0.932} & 0.898 & 0.830 & 0.875 \\
FacesUCR & 0.924 & \textbf{0.928} & 0.835 & 0.918 \\
FiftyWords & 0.771 & \textbf{0.785} & 0.745 & 0.780 \\
Fish & 0.926 & 0.949 & \textbf{0.960} & 0.880 \\
FordA & \textbf{0.936} & \textbf{0.936} & 0.927 & 0.935 \\
FordB & 0.794 & 0.779 & 0.798 & \textbf{0.810} \\
GunPoint & 0.980 & \textbf{0.993} & 0.987 & \textbf{0.993} \\
Ham & \textbf{0.714} & \textbf{0.714} & 0.533 & 0.695 \\
HandOutlines & \textbf{0.922} & 0.919 & 0.919 & \textbf{0.922} \\
Haptics & \textbf{0.526} & \textbf{0.526} & 0.474 & 0.455 \\
Herring & \textbf{0.641} & 0.594 & 0.578 & 0.578 \\
InlineSkate & 0.415 & \textbf{0.465} & 0.444 & 0.447 \\
InsectWingbeatSound & \textbf{0.630} & 0.603 & 0.599 & 0.623 \\
ItalyPowerDemand & 0.925 & \textbf{0.957} & 0.929 & 0.925 \\
LargeKitchenAppliances & 0.845 & \textbf{0.861} & 0.765 & 0.848 \\
Lightning2 & 0.869 & \textbf{0.918} & 0.787 & \textbf{0.918} \\
Lightning7 & \textbf{0.863} & 0.781 & 0.740 & 0.795 \\
Mallat & 0.914 & 0.956 & 0.916 & \textbf{0.964} \\
Meat & 0.950 & \textbf{0.967} & 0.867 & 0.950 \\
MedicalImages & \textbf{0.789} & 0.784 & 0.725 & 0.784 \\
MiddlePhalanxOutlineCorrect & \textbf{0.838} & 0.794 & 0.787 & 0.814 \\
MiddlePhalanxOutlineAgeGroup & 0.636 & 0.649 & 0.623 & \textbf{0.656} \\
MiddlePhalanxTW & 0.584 & 0.597 & 0.584 & \textbf{0.610} \\
MoteStrain & 0.861 & 0.847 & 0.823 & \textbf{0.871} \\
NonInvasiveFetalECGThorax1 & 0.930 & \textbf{0.946} & 0.925 & 0.910 \\
NonInvasiveFetalECGThorax2 & 0.938 & \textbf{0.955} & 0.930 & 0.927 \\
OliveOil & \textbf{0.900} & \textbf{0.900} & \textbf{0.900} & \textbf{0.900} \\
OSULeaf & 0.851 & \textbf{0.868} & 0.736 & 0.831 \\
PhalangesOutlinesCorrect & \textbf{0.809} & 0.794 & 0.784 & 0.801 \\
Phoneme & \textbf{0.312} & 0.260 & 0.196 & 0.289 \\
Plane & \textbf{1.000} & 0.981 & 0.981 & 0.990 \\
ProximalPhalanxOutlineCorrect & \textbf{0.887} & 0.876 & 0.869 & 0.859 \\
ProximalPhalanxOutlineAgeGroup & 0.834 & 0.844 & 0.839 & \textbf{0.854} \\
ProximalPhalanxTW & \textbf{0.824} & 0.805 & 0.785 & \textbf{0.824} \\
RefrigerationDevices & \textbf{0.589} & 0.557 & 0.555 & 0.517 \\
ScreenType & 0.411 & \textbf{0.421} & 0.384 & 0.413 \\
ShapeletSim & \textbf{1.000} & \textbf{1.000} & 0.517 & 0.817 \\
ShapesAll & \textbf{0.902} & 0.877 & 0.837 & 0.875 \\
SmallKitchenAppliances & 0.731 & \textbf{0.747} & 0.731 & 0.715 \\
SonyAIBORobotSurface1 & \textbf{0.903} & 0.884 & 0.840 & 0.897 \\
SonyAIBORobotSurface2 & 0.871 & 0.872 & 0.832 & \textbf{0.934} \\
StarLightCurves & \textbf{0.969} & 0.967 & 0.968 & 0.965 \\
Strawberry & \textbf{0.962} & \textbf{0.962} & 0.946 & 0.946 \\
SwedishLeaf & \textbf{0.941} & 0.931 & 0.925 & 0.931 \\
Symbols & \textbf{0.976} & 0.973 & 0.945 & 0.965 \\
SyntheticControl & \textbf{0.997} & \textbf{0.997} & 0.977 & 0.983 \\
ToeSegmentation1 & 0.917 & 0.947 & 0.899 & \textbf{0.952} \\
ToeSegmentation2 & 0.892 & \textbf{0.946} & 0.900 & 0.885 \\
Trace & \textbf{1.000} & \textbf{1.000} & \textbf{1.000} & \textbf{1.000} \\
TwoLeadECG & 0.986 & \textbf{0.999} & 0.993 & 0.997 \\
TwoPatterns & \textbf{1.000} & 0.999 & 0.992 & \textbf{1.000} \\
UWaveGestureLibraryX & 0.795 & \textbf{0.818} & 0.784 & 0.811 \\
UWaveGestureLibraryY & 0.719 & \textbf{0.739} & 0.697 & 0.735 \\
UWaveGestureLibraryZ & \textbf{0.770} & 0.757 & 0.729 & 0.759 \\
UWaveGestureLibraryAll & 0.930 & 0.918 & 0.865 & \textbf{0.941} \\
Wafer & \textbf{0.998} & 0.997 & 0.995 & 0.993 \\
Wine & \textbf{0.870} & 0.759 & 0.685 & \textbf{0.870} \\
WordSynonyms & 0.676 & 0.693 & 0.641 & \textbf{0.704} \\
Worms & 0.701 & \textbf{0.753} & 0.688 & 0.714 \\
WormsTwoClass & 0.805 & 0.688 & 0.753 & \textbf{0.818} \\
Yoga & \textbf{0.887} & 0.855 & 0.828 & 0.878 \\
\midrule
    AVG & \textbf{0.829} & 0.824 & 0.786 & 0.821 \\
    Rank & \textbf{2.041} & 2.188 & 3.417 & 2.352 \\
    \bottomrule
  \caption{Accuracy scores of our method compared with those of T-Loss-4X, T-Loss$^\dag$ on the first 85 UCR datasets.}\label{full-ucr-other}
\end{longtable}

\normalsize

\twocolumn

\end{document}